\documentclass[11pt,reqno,titlepage]{amsart}
\usepackage{tikz}
\usepackage{essay-def-arxiv}
\usepackage{subfig}
\usepackage{epstopdf}
\usepackage{epsfig}
\usepackage{lipsum}
\usepackage{hyperref}
\usepackage{url}
\usepackage{amsmath}
\usepackage{amsfonts}
\usepackage{lastpage}
\usepackage{graphicx}
\usepackage{multirow}
\usepackage{tabu}
\usepackage{clrscode}
\usepackage{caption}
\usepackage{appendix}
\usepackage{mathrsfs}
\usepackage{algorithm}
\usepackage{algorithmic}
\usepackage{natbib}
\usepackage{tikz-cd}
\usepackage{stackengine}
\setstackEOL{\\}
\usepackage[foot]{amsaddr}

\usepackage{tikz}
\usetikzlibrary{shapes,arrows}

\newcommand{\rank}{\mathrm{rank}}
\newcommand{\DKL}{\mathrm{D}_\mathrm{KL}}

\newcommand{\titlemark}{Optimal Neural Network Approximation of Wasserstein Gradient Direction via Convex Optimization}
\title{\textbf{\titlemark}}

\usetikzlibrary{positioning}
\tikzset{main node/.style={circle,fill=blue!20,draw,minimum size=1cm,inner sep=0pt},  }

\makeatletter
\g@addto@macro{\endabstract}{\@setabstract}
\newcommand{\authorfootnotes}{\renewcommand\thefootnote{\@fnsymbol\c@footnote}}%
\makeatother

\begin{document}

\begin{center}
  \LARGE 
  \titlemark
  \par \bigskip

  \normalsize
  \authorfootnotes
  Yifei Wang\footnote{\texttt{wangyf18@stanford.edu}}\textsuperscript{1},
  Peng Chen \footnote{\texttt{peng@oden.utexas.edu}}\textsuperscript{2},
  Mert Pilanci \footnote{\texttt{pilanci@stanford.edu}}\textsuperscript{1} and 
  Wuchen Li\footnote{\texttt{wuchen@mailbox.sc.edu}}\textsuperscript{3}
 \par \bigskip
 \textsuperscript{1}Department of Electrical Engineering, Stanford University \par
  \textsuperscript{2} Oden Institute for Computational Engineering and Sciences, The University of Texas at Austin\par
   \textsuperscript{3}Department of Mathematics, University of South Carolina\par \bigskip
\end{center}

\begin{abstract}
   The computation of Wasserstein gradient direction is essential for posterior sampling problems and scientific computing. The approximation of the Wasserstein gradient with finite samples requires solving a variational problem. We study the variational problem in the family of two-layer networks with squared-ReLU activations, towards which we derive a semi-definite programming (SDP) relaxation. This SDP can be viewed as an approximation of the Wasserstein gradient in a broader function family including two-layer networks. By solving the convex SDP, we obtain the optimal approximation of the Wasserstein gradient direction in this class of functions. Numerical experiments including PDE-constrained Bayesian inference and parameter estimation in COVID-19 modeling demonstrate the effectiveness of the proposed method.
\end{abstract}

\section{Introduction}
% Receiving intense interests across the field of inverse problems, 
Bayesian inference plays essential roles in learning model parameters from the observational data with applications in inverse problems, scientific computing, information science, and machine learning \citep{ivabp}.
%\citep{ivabp, SVGD}. 
% It quantifies the uncertainty of a system formulated by combining complicated models and data. 
The central problem in Bayesian inference is to draw samples from a posterior distribution, which characterizes the parameter distribution
% of a system
given data and a prior distribution. % From an optimization and mean-field perspective, with the number of drawn samples going to infinity, this problem is equivalent to minimizing an objective function in the space of probability density. The objective functional, for instance, Kullback-Leibler (KL) divergence evaluates the similarity of the current distribution and the posterior distribution. 

The Wasserstein gradient flow \citep{tgode, ambrosio2005gradient,junge2017fully} has shown to be effective in drawing samples from a posterior distribution, which have attracted increasing attention in recent years. For instance, the Wasserstein gradient flow of Kullback-Leibler (KL) divergence connects to the overdampled Langevin dynamics. %is Fokker-Planck equation of the overdamped Langevin dynamics. 
The time-discretization of the overdamped Langevin dynamics renders the classical Langevin Monte Carlo Markov Chain (MCMC) algorithm. In this sense, the computation of Wasserstein gradient flow yields a different viewpoint for sampling algorithms.  %  and proximal Langevin algorithms \citep{lmcaj, plarc}. 
In particular, the Wasserstein gradient direction also provides a deterministic update of the particle system \citep{carrillo2021lagrangian}. %We can formulate the Wasserstein gradient direction as the optimal solution to a variational problem. %also has a variational formulation, which corresponds to the optimal solution to a variational problem. We refer to this problem as the Wasserstein gradient variational problem. The key difficulty in solving the variational problem is that in practice we only access to finite many samples from the current distribution instead of the entire distribution. 
%  By inexactly solving the variational problem in different functional subspaces, 
Based on the approximation or generalization of the Wasserstein gradient direction, 
many efficient sampling algorithms have been developed, including Wasserstein gradient descent (WGD) with kernel density estimation (KDE) \citep{liu2019understanding}, %{\color{red} a citation of Chang Liu or others' paper here?},
Stein variational gradient descent (SVGD) \citep{SVGD}, and neural variational gradient descent \citep{di2021neural}, etc. 
% Recent works on projection methods \citep{chen2020projected, wang2021projected} empirically improve the performance of WGD with KDE and SVGD by reducing the dimensionality of the parameter. 

Meanwhile, neural networks exhibit tremendous %and astonishing 
optimization and generalization performance in learning complicated functions from data. They also have wide applications in Bayesian inverse problems \citep{rezende2015variational, onken2020ot, kruse2019hint, lan2021scaling}. According to the universal approximation theorem of neural networks \citep{hornik1989multilayer,lu2017expressive}, any arbitrarily complicated functions can be learned by a two-layer neural network with non-linear activations and a sufficient number of neurons. %Solving the Wasserstein gradient variational problem in the family of neural networks 
Functions represented by neural networks naturally provide an approximation towards the Wasserstein gradient direction.

However, due to the nonlinear and nonconvex structure of neural networks, optimization algorithms including stochastic gradient descent may not find the global optima of the training problem. Recently, based on a line of works \citep{pilanci2020neural, sahiner2020vector, bartan2021neural}, the regularized training problem of two-layer neural networks with ReLU/polynomial activation %and a convex loss function 
can be formulated as a convex program. The optimal solution of the convex program renders a global optimum of the nonconvex training problem. %, which can be globally optimized. %A natural question arises: what is the global optimal function which solves the Wasserstein gradient variational problem with regularization in the family of two-layer neural networks? 

In this paper, we study a variational problem, whose optimal solution corresponds to the Wasserstein gradient direction. Focusing on the family of two-layer neural networks with squared ReLU activation, we formulate the regularized variational problem in terms of samples. Directly training the neural network to miminize the loss may get the neural network stuck at local minima or saddle points and it often leads to biased sample distribution from the posterior.  %We present a convex SDP relaxation for the regularized variational problem with respect to two-layer neural networks with squared ReLU activation. %Following the convex optimization formulation of two-layer networks, 
Instead, we analyze the convex dual problem of the training problem and study its semi-definite program (SDP) relaxation by analyzing the geometry of dual constraints. The resulting SDP is practically solvable and it can be efficiently optimized by convex optimization solvers such as CVXPY \citep{diamond2016cvxpy}. We then derive the corresponding relaxed bidual problem (dual of the relaxed dual problem). Thus, the optimal solution of the dual problem yields an optimal approximation of the Wasserstein gradient direction in a broader function family. We also present practical implementation and analyze the choice of the regularization parameter. Numerical results including PDE-constrained inference problems and Covid-19 parameter estimation problems illustrate the effectiveness of our proposed method.

\subsection{Related works}
The time and spatial discretizations of Wasserstein gradient flows % with respect to various energy functionals including KL divergence 
are extensively studied in literature \citep{jordan1998variational,junge2017fully,carrillo2021primal, carrillo2021lagrangian, bonet2021sliced, liutkus2019sliced, frogner2020approximate}.  %Recently, \citet{mokrov2021large} apply input-convex neural networks to approximate the Wasserstein gradient flow. % with applications of input-convex neural networks \citep{mokrov2021large}. %\citet{carrillo2021lagrangian} provide a systematic review of time-discretization of Wasserstein gradient flow from the perspective of the Lagrangian scheme. Neural networks also  
Recently, neural networks have been applied in solving or approximating %have many applications in solving partial differential equations (PDE) related to 
Wasserstein gradient flows \citep{mokrov2021large, lin2021wasserstein, lin2021alternating,alvarez2021optimizing, bunne2021jkonet, hwang2021deep, fan2021variational}.
For sampling algorithms, % based on the approximation of the Wasserstein gradient flow, 
\citet{di2021neural} learns the transportation function by solving an unregularized variational problem in the family of vector-output deep neural networks. Compared to these studies, we focus on a convex SDP relaxation of the varitional problem induced by the Wasserstein gradient direction. Meanwhile, \citet{feng2021relative} form the Wasserstein gradient direction as the mininimizer the Bregman score and they apply deep neural networks to solve the induced variational problem.  %variational problem with regularization 

\section{Background}
In this section, we briefly review the Wasserstein gradient descent, and present its variational formulation. In particular, we focus on the Wasserstein gradient descent direction of KL divergence functional. Later on, we design a neural network convex optimization problems to approximate Wasserstein gradient in samples. 

%and the convex optimization formulation of two-layer ReLU neural networks.
\subsection{Wasserstein gradient descent}
% Let $\Omega\subseteq \mbR^d$ be a region. 
Consider an optimization problem in the probability space:
\begin{equation}\label{prob:kl}
\inf_{\rho\in \mcP}\;%\mbE(\rho)=
\DKL(\rho\|\pi)=\int \rho (\log\rho-\log \pi)dx,    
\end{equation}
Here the integral is taken over $\mbR^d$ and the objective functional $\DKL(\rho\|\pi)$ is the KL divergence from $\rho$ to $\pi$. The variable is the density function $\rho$ in the space $\mcP=\{\rho\in C^\infty(\mbR^d)|\int \rho dx =1,~\rho> 0\}$. The function $\pi\in C^\infty(\mbR^d)$ is a known probability density function of the posterior distribution. By solving the optimization problem \eqref{prob:kl}, we can generate samples from the posterior distribution.

%To be specific, we consider the objective functional $\mbE(\rho)$ as the Kullback-Leibler divergence  from $\rho$ to $\pi$, i.e.,  $\DKL(\rho\|\pi)=\int \rho (\log\rho-\log \pi)dx$. 

A known fact \cite[Chapter 8.3.1]{tiot} is that the Wasserstein gradient descent flow for the optimization problem \eqref{prob:kl} satisfies 
\begin{equation*}
\begin{split}
\partial_t\rho_t=&\nabla\cdot\pp{\rho_t\nabla\frac{\delta}{\delta \rho_t}\mathrm{D}_{\mathrm{KL}}(\rho_t\|\pi)}\\
= &\nabla\cdot(\rho_t(\nabla \log \rho_t-\nabla \log \pi))\\
    =&\Delta \rho_t-\nabla\cdot(\rho_t\nabla\log \pi),
\end{split}
\end{equation*}
where $\rho_t(x)=\rho(x,t)$ and $\frac{\delta}{\delta\rho_t}$ is the $L^2$ first variation operator w.r.t. $\rho_t$. In above third equality, a fact $\rho_t\nabla\log\rho_t=\nabla\rho_t$ is used. Here $\nabla\cdot F$ denotes the divergence of a vector valued function $F:\mbR^d\to\mbR^d$ and $\Delta$ is the Laplace operator. This equation is also known as the gradient drift Fokker-Planck equation. It corresponds to the following updates in terms of samples:
\begin{equation}\label{wgd:sample}
    dx_t = -(\nabla \log \rho_t(x_t)-\nabla \log \pi(x_t))dt,
\end{equation}
where $x_t$ follows the distribution of $\rho_t$. Clearly, when $\rho_t=\pi$, the above dynamics reach the equilibrium,  which implies that the samples $x_t$ are generated by the posterior distribution.  %Suppose that $\rfollowsho_t$ is known in advance, then the dynamics \eqref{wgd:sample} evolves the particles toward the posterior distribution $\pi$. 
%The above formulation is also known as the Lagrangian formulation of the evolution of the Fokker-Planck equation.  
%\citep{junge2015fully} or the Lagrangian scheme of the Wasserstein gradient flow \citep{carrillo2021lagrangian}. 

To solve the Wasserstein gradient flow \eqref{wgd:sample}, we consider a forward Eulerian discretization in time. 
%In discrete-time, 
In the $l$-th iteration, suppose that $\{x_l^n\}$ are samples drawn from $\rho_l$. The update rule of Wasserstein gradient descent (WGD) on the particle system $\{x_l^n\}$ follows
\begin{equation}\label{upd:wgd}
    x_{l+1}^n = x_l^n-\alpha_l\nabla \Phi_l(x_l^n),
\end{equation}
where $\Phi_l:\mbR^d\to \mbR$ is a function which approximates $\log\rho_l-\log\pi$. 

\subsection{Variational formulation of WGD}
Given the particles $\{x_n\}_{n=1}^N$, we design the following variational problem to choose a suitable function $\Phi$ approximating the function $\log \rho -\log \pi$. Consider
\begin{equation}\label{equ:wgd}
\inf_{\Phi\in C^1(\mbR^d)} \frac{1}{2}\int \|\nabla \Phi-(\nabla \log \rho-\nabla \log \pi)\|_2^2 \rho dx.
\end{equation}
The objective functional evaluates the least-square discrepancy between $\nabla \log\rho-\nabla \log \pi$ and $\nabla \Phi$ weighted by the density $\rho$. %in estimating $\nabla \Phi-\nabla \log \rho$ using $\nabla \Phi$. 
The optimal solution follows $\Phi=\log \rho -\log \pi$, up to a constant shift. %However, the space of smooth function is too large to tackle with. 
Let $\mcH\subseteq C^1(\mbR^d)$ be a finite dimensional function space. The following proposition gives a formulation of \eqref{equ:wgd} in $\mcH$.
\begin{proposition}\label{prop:wgd_var}
Let $\mcH\subseteq C^1(\mbR^d)$ be a function space. The variational problem \eqref{equ:wgd} in the domain $\mcH$ is equivalent to
\begin{equation}\label{equ:wgd_h}
    \inf_{\Phi\in \mcH} \frac{1}{2} \int \|\nabla \Phi\|_2^2 \rho dx+\int \Delta \Phi \rho dx+\int \lra{\nabla \log\pi,\nabla \Phi}\rho dx.
\end{equation}
\end{proposition}
\begin{remark}
A similar variational problem has been studied in \citep{di2021neural}. If we replace $\nabla \Phi$ for $\Phi\in \mcH$ by a vector field $\Psi$ in certain function family, then, the quantity in \eqref{equ:wgd_h} is the negative regularized Stein discrepancy defined in \citep{di2021neural} between $\rho$ and $\pi$ based on $\Psi$. 
\end{remark}
%Assuming that the functional space $\mcH$ is rich enough, then the optimal solution to \eqref{equ:wgd} is $\Phi=\log \rho-\log \pi$. 
% We can write the objective function in \eqref{equ:wgd_h} as
% \begin{equation}\label{obj:var}
%     \frac{1}{2}\mbE_{X\sim \rho}[ \|\nabla \Phi(X)\|_2^2+\lra{\nabla \Phi(X),\nabla \log\pi(X)}+\Delta \Phi(X)].
% \end{equation}
Therefore, by replacing the density $\rho$ by finite samples $\{x_n\}_{n=1}^N\sim\rho$, the problem \eqref{equ:wgd_h} in terms of finite samples forms
\begin{equation}\label{equ:wgd_h_smp}
\begin{aligned}
\inf_{\Phi\in \mcH} & \frac{1}{N} \sum_{n=1}^N\pp{\frac{1}{2}\|\nabla \Phi(x_n)\|_2^2 +\Delta \Phi(x_n)}+\frac{1}{N} \sum_{n=1}^N\lra{\nabla \log \pi(x_n),\nabla \Phi(x_n)}.
\end{aligned}
\end{equation}

\section{Optimal neural network approximation of Wasserstein gradient}
In this section, we focus on functional space $\mathcal{H}$ of functions represented by two-layer neural networks. We derive the primal and dual problem of the regularized Wasserstein variational problems. By analyzing the dual constraints, a convex SDP relaxation of the dual problem is obtained. We also present a practical implementation estimation of $\nabla \log\rho-\nabla \log \pi$ and discuss the choice of the regularization parameter. 

Let $\psi$ be an activation function. Consider the case where $\mcH$ is a class of two-layer neural network with the activation function $\psi(x)$:
\begin{equation}\label{eq:twolayerNN}
\mcH=\bbbb{\Phi_{\btheta}\in C^{1}(\mathbb{R}^d)|\Phi_{\btheta}(x)%{\color{red}=}
=\alpha^T\psi(W^Tx)},
\end{equation}
where $\btheta = (W,\alpha)$ is the parameter in the neural network with $W\in \mbR^{d\times m}$ and $\alpha\in \mbR^m$. %{\color{red} shape of $W$?} 
\begin{remark}
We can extend this model to handle the bias term by add an entry of $1$ in $x_1,\dots,x_n$. 
\end{remark}
For two-layer neural networks, we can compute the gradient and Laplacian of $\Phi\in \mcH$ as follows:
\begin{equation}
\nabla \Phi_{\btheta}(x)  =  \sum_{i=1}^m\alpha_iw_i\psi'(w_i^Tx)= W(\psi'(W^Tx)\circ \alpha),
\end{equation}
\begin{equation}\label{equ:laplace}
\Delta \Phi_{\btheta}(x) = \sum_{i=1}^m \alpha_i \|w_i\|_2^2 \psi''(w_i^Tx).
\end{equation}
% Thus, we can rewrite the variational problem \eqref{equ:wgd_h_smp} as
Here $\circ$ represents the element-wise multiplication. By adding a regularization term to the variational problem \eqref{equ:wgd_h_smp}, we obtain
\begin{equation}\label{equ:wgd_reg}
\begin{aligned}
\min_{\btheta } &\frac{1}{2N} \sum_{n=1}^N\norm{\sum_{i=1}^m\alpha_iw_i\psi'(w_i^Tx_n)}_2^2+\frac{1}{N} \sum_{n=1}^N \lra{\sum_{i=1}^m\alpha_iw_i\psi'(w_i^Tx_n), \nabla  \log \pi(x_n)}\\
&+\frac{1}{N} \sum_{n=1}^N\sum_{i=1}^m \alpha_i \|w_i\|_2^2 \psi''(w_i^Tx_n) +\frac{\beta}{2} R(\btheta),
\end{aligned}
\end{equation}
where $\beta>0$ is the regularization parameter. 
% \begin{equation}
% \begin{aligned}
% \min_{\btheta } &\frac{1}{2N} \sum_{n=1}^N\norm{\sum_{i=1}^m\alpha_iw_i\psi'(w_i^Tx_n)}_2^2\\
% &+\frac{1}{N} \sum_{n=1}^N\sum_{i=1}^m \alpha_i \|w_i\|_2^2 \psi''(w_i^Tx_n) \\
% &+ \frac{1}{N}\sum_{n=1}^N \lra{\sum_{i=1}^m\alpha_iw_i\psi'(w_i^Tx_n), \nabla \log\pi(x_n)}.
% \end{aligned}
% \end{equation}
% \begin{remark}
% One potential concern is that if we simply choose $\psi(z)=(z)_+^2$, then $\psi''(w_i^Tx_n)$ may not be well-defined when $w_i^Tx_n=0$. %This might be criticized by reviewers. However, if we choose $\psi(z)=(z)_+^3$, then we will have the cube of entries of $w_i$ in the variational problem, which makes the problem extremely tough to tackle with. 
% \end{remark}
We focus on the squared ReLU activation $\psi(z)=(z)_+^2 = \pp{\max\{z,0\}}^2$. Note that a non-vanishing second derivative is required for the Laplacian term in \eqref{equ:laplace}, which makes the ReLU activation inadequate.  
% We focus on the square activation $\psi(z)=z^2$ and the squared ReLU activation $\psi(z)=(z)_+^2$. %the activation functions satisfying $\psi(az)=a^2z$ for $a>0$. 
For this activation function, we consider the following regularization function
\begin{equation}\label{reg_term}
R(\btheta) = \sum_{i=1}^m(\|w_i\|_2^3+|\alpha_i|^3).
\end{equation}
\begin{remark}
We note that $\nabla \Phi_{\btheta}(x)$ and $\Delta \Phi_{\btheta}(x)$ are all piece-wise degree-$3$ polynomials of the parameters $\btheta$. Hence, we consider a specific cubic regularization term \eqref{reg_term}, analogous to \citep{bartan2021neural}. By choosing this regularization term, we can derive a simplified convex dual problem. 
\end{remark}

By rescaling the first and second-layer parameters, the regularized variational problem \eqref{equ:wgd_reg} can be formulated as follows.
\begin{proposition}[Primal problem]\label{prop:reform}
The regularized variational problem \eqref{equ:wgd_reg} is equivalent to
\begin{equation}\label{equ:wgd_reg_rep}
\begin{aligned}
\min_{W,\alpha}\; & \frac{1}{2}\sum_{n=1}^N\norm{\sum_{i=1}^m\alpha_iw_i\psi'(w_i^Tx_n)}^2+\sum_{n=1}^N\sum_{i=1}^m \alpha_i \|w_i\|_2^2 \psi''(w_i^Tx_n) \\
&+ \sum_{n=1}^N \lra{\sum_{i=1}^m\alpha_iw_i\psi'(w_i^Tx_n), \nabla \log \pi(x_n)}+\tilde \beta\|\alpha\|_1,\\
\text{ s.t. } &\|w_i\|_2\leq 1, i\in[m],
\end{aligned}
\end{equation}
where $\tilde \beta=3\cdot 2^{-5/3}N\beta$. 
\end{proposition}

For simplicity, we write $Y=\bmbm{\nabla\log \pi(x_1)^T\\\vdots\\\nabla \log \pi(x_N)^T }\in \mbR^{N\times d}$. We introduce the slack variable $z_n=\sum_{i=1}^m\alpha_iw_i\psi'(x_n^Tw_i)$ for $n\in[N]$ and denote
$Z=\bmbm{z_1&\dots&z_N}^T\in\mbR^{N\times d}$.
%$Z=\bmbm{z_1^T\\\vdots\\z_N^T}\in\mbR^{N\times d}$. 
Then, we can simplify the problem \eqref{equ:wgd_reg_rep} to %{(\color{red} why min w.r.t.\ $Z$?)}
\begin{equation}\label{cvx_nn:p}
\begin{aligned}
\min_{W,\alpha,Z}\; &\frac{1}{2}\|Z\|_F^2+\sum_{n=1}^N\sum_{i=1}^m \alpha_i \|w_i\|_2^2 \psi''(w_i^Tx_n) +\tr(Y^TZ)+\tilde \beta \|\alpha\|_1,\\
\text{ s.t. } & z_n=\sum_{i=1}^m\alpha_iw_i\psi'(x_n^Tw_i), n\in[N], \|w_i\|_2\leq 1, i\in[m].
\end{aligned}
\end{equation}

Based on the above reformulation, we can derive the dual problem of \eqref{cvx_nn:p} as follows.
\begin{proposition}[Dual problem]\label{prop:dual}
The dual problem of the regularized variational problem \eqref{cvx_nn:p} is %{\color{red} relation between $m$ and $N$?}
\begin{equation}\label{cvx_nn:d}
\begin{aligned}
\max_{\Lambda\in \mbR^{N\times d}} \; &-\frac{1}{2}\|\Lambda+Y\|_F^2,\text{ s.t. }\max_{w:\|w\|_2\leq 1}\left|\sum_{n=1}^N \|w\|_2^2\psi''(x_n^Tw)-\lambda_n^Tw\psi'(x_n^Tw)\right|\leq \tilde \beta,
\end{aligned}
\end{equation}
which provides a lower-bound on \eqref{cvx_nn:p}.

\end{proposition}

\subsection{Analysis of dual constraints and the relaxed dual problem}

Now, we analyze the constraint 
$$
\max_{w:\|w\|_2\leq 1}\left|\sum_{n=1}^N \|w\|_2^2\psi''(w^Tx_n)-\lambda_n^Tw\psi'(x_n^Tw)\right|\leq \tilde \beta
$$
in the dual problem. We note that this constraint is closely related to the regularization parameter, which we will discuss later.  
% We focus on the squared ReLU activation $\psi(z)=(z)_+^2$. 
For simplicity, we take $\psi''(0)=0$ as the subgradient of $\psi'(z)$ at $z=0$, i.e., taking the left derivative of $\psi'(z)$ at $z=0$. %since it is not twice differentiable at $0$.% {\color{red} is this a choice or given by definition?}. 
Let $X = [x_1, \dots, x_N]^T\in\mbR^{N\times d}$. 
%Let $X=\bmbm{x_1^T\\\vdots\\x_N^T}\in\mbR^{N\times d}$. % {\color{red} $n$ or $N$? how about $X = (x_1, \dots, x_N)^T$?} 
Denote the set of all possible hyper-plane arrangements corresponding to the rows of $X$ as %{\color{red} what is $\mbI(Xw)$?}
\begin{equation}\label{equ:s}
    \mcS=\{D=\diag(\mbI(Xw\geq 0))|w\in\mbR^d, w\neq 0\}.
\end{equation}
Here $\mbI(s)=1$ if the statement $s$ is correct and $\mbI(s)=0$ otherwise. Let $p=|\mcS|$ be the cardinality of $\mcS$, % {\color{red} the cardinality of $\mcS$}, 
and write $\mcS=\{D_1,\dots,D_p\}$. According to \citep{cover1965geometrical}, we have the upper bound $p \leq 2r \left(\frac{e (N-1)}{r}\right)^r$, where $r = \text{rank}(X)$. % {\color{red} I suggest to move all these matrices after (16), with proper explanation of each term and its property, e.g., symmetric, positive or not.}

% Based on the dual problem of \eqref{sdp:dual}, we can derive an SDP as a relaxed dual problem. 
Based on the analysis of the dual constraints, we can derive a convex SDP as a relaxed dual problem. It gives a lower bound for the optimal value of the dual problem \eqref{cvx_nn:d}.

\begin{proposition}[Relaxed Dual problem]\label{prop:dual_relax}
Consider the following SDP: %{\color{red} explain why it is SDP by stating the properties of the matrices}
\begin{equation}\label{cvx_nn:d_relax}
\begin{aligned}
\max %_{\Lambda\in \mbR^{N\times d}}
\;& -\frac{1}{2}\|\Lambda+Y\|_F^2,\\
\text{ s.t. }
&\tilde A_j(\Lambda)+\tilde B_j+\sum_{n=0}^N r_n^{(j,-)} H^{(j)}_n+\tilde \beta e_{d+1}e_{d+1}^T\succeq 0,\\
&-\tilde A_j(\Lambda)-\tilde B_j+\sum_{n=0}^N r_n^{(j,+)} H^{(j)}_n+\tilde \beta e_{d+1}e_{d+1}^T\succeq 0,\\
&r^{(j,-)}\geq 0, r^{(j,+)}\geq 0, j\in[p].
\end{aligned}
\end{equation}
The variables are $\Lambda\in\mbR^{N\times d}$ and $r^{(j,-)},r^{(j,+)}\in\mbR^{n+1}$ for $j\in[p]$. For $j\in[p]$, we denote $A_j(\Lambda)=-\Lambda^TD_jX-X^TD_j\Lambda$, $B_j=2\tr(D_j)I_d$, $\tilde A_j(\Lambda) = \bmbm{A_j(\Lambda)&0\\0&0},\tilde B_j = \bmbm{B_j&0\\0&0}$, $H_0^{(j)}=\bmbm{I_d&0\\0&-1}$ and $H^{(j)}_n=\bmbm{0&(1-2(D_j)_{nn})x_n\\(1-2(D_j)_{nn})x_n^T&0}, n\in[N]$
The vector $e_{d+1}\in\mbR^{d+1}$ satisfies that $(e_{d+1})_i=0$ for $i\in[d]$ and $(e_{d+1})_{d+1}=1$. 

The optimal value of \eqref{cvx_nn:d_relax} gives a lower bound on the dual problem \eqref{cvx_nn:d}, and hence on the primal problem \eqref{cvx_nn:p}.
\end{proposition}

In the following proposition, we derive the relaxed bi-dual problem.  It can be viewed as a convex relaxation of the primal problem \eqref{cvx_nn:p}. 
\begin{proposition}[Relaxed bi-dual problem]\label{prop:bidual}
The dual of the relaxed dual problem \eqref{cvx_nn:d_relax} is as follows %{\color{red} min w.r.t.\ which variables?}
\begin{equation}\label{cvx_nn:bidual}
\begin{aligned}
\min %_{Z\in \mbR^{N\times d}}
\;& \frac{1}{2} \|Z+Y\|_F^2-\frac{1}{2}\|Y\|_F^2+\sum_{j=1}^p\tr(\tilde B_j(S^{(j,+)}-S^{(j,-)}))+\tilde \beta\sum_{j=1}^p\tr\pp{(S^{(j,+)}+S^{(j,-)})e_{d+1}e_{d+1}^T },\\
\text{ s.t. }&Z =\sum_{j=1}^p \tilde A_j^*(S^{(j,-)}-S^{(j,+)}), \tr(S^{(j,-)}H_n^{(j)})\leq 0, \tr(S^{(j,+)}H_n^{(j)})\leq 0, n=0,\dots,N, j\in[p],
\end{aligned}
\end{equation}
in variables $Z\in \mbR^{N\times d}$, $S^{(j,+)},S^{(j,-)}\in\mbS^{d+1}_+$ for $j\in[p]$.  Here $A_j^*$ is the adjoint operator of the linear operator $A_j$. %{\color{red} what does it mean ``can be mapped"? what is the map? can you or not explicitly construct it? what does ``feasible" mean? local/global minimum?}
\end{proposition}
As \eqref{cvx_nn:d_relax} is a convex problem and the Slater's condition is satisfied, the optimal values of \eqref{cvx_nn:d_relax} and \eqref{cvx_nn:bidual} are same. 
We can show that any feasible solutions of the primal problem \eqref{equ:wgd_reg_rep} can be mapped to feasible solutions of \eqref{cvx_nn:bidual}.
\begin{theorem}\label{thm:map}
Suppose that $(Z,W,\alpha)$ is feasible to the primal problem \eqref{cvx_nn:p}. Then, there exist matrices $\{S^{(j,+)},S^{(j,-)}\}_{j=1}^p$ constructed from $(W, \alpha)$ such that $(Z,\{S^{(j,+)},S^{(j,-)}\}_{j=1}^p)$ is feasible to %a solution of
% {\color{red} is it OK to replace ``feasible for" by ``a solution of"? or more specifically, does it mean that the solution of the primal problem can be used to construct a solution of the relaxed bi-dual problem?} 
the relaxed bi-dual problem \eqref{cvx_nn:bidual}. Moreover, the objective value of the relaxed bi-dual problem \eqref{cvx_nn:bidual} at $(Z,\{S^{(j,+)},S^{(j,-)}\}_{j=1}^p)$ is the same as objective value of the primal problem \eqref{cvx_nn:p} at $(Z,W,\alpha)$. % they have the same objective value.
\end{theorem}
Let $J(Z,\{S^{(j,+)},S^{(j,-)}\}_{j=1}^p)$ denote the objective value of the relaxed bi-dual problem \eqref{cvx_nn:bidual} at a feasible solution $(Z,\{S^{(j,+)},S^{(j,-)}\}_{j=1}^p)$. Let $( Z^*,W^*,\alpha^*)$ denote a globally optimal solution of the primal problem \eqref{cvx_nn:p}. By Theorem \ref{thm:map}, there exist matrices $\{S^{(j,+)},S^{(j,-)}\}_{j=1}^p$ such that $(Z^*,\{S^{(j,+)},S^{(j,-)}\}_{j=1}^p)$ is a feasible solution of the relaxed bi-dual problem \eqref{cvx_nn:bidual} and $J(Z^*,\{S^{(j,+)},S^{(j,-)}\}_{j=1}^p)$ is the same as the objective value of \eqref{cvx_nn:p} at its global minimum $(Z^*,W^*,\alpha^*)$. On the other hand, let $(\tilde Z^*,\{\tilde S^{(j,+)},\tilde S^{(j,-)}\}_{j=1}^p)$ denote an optimal solution of the relaxed bi-dual problem \eqref{cvx_nn:bidual}. %{\color{red} how many solutions are there for the relaxed bi-dual problem? is it not unique?}
From the optimality of $(\tilde Z^*,\{\tilde S^{(j,+)},\tilde S^{(j,-)}\}_{j=1}^p)$, we have
\begin{equation}
    J(\tilde Z^*,\{\tilde S^{(j,+)},\tilde S^{(j,-)}\}_{j=1}^p)\leq J(Z^*,\{S^{(j,+)},S^{(j,-)}\}_{j=1}^p).
\end{equation}
%Based on the convex duality, we present a characterization of the optimal approximation $\nabla \Phi(x)$ in terms of the optimal solution to the relaxed dual problem \eqref{cvx_nn:d_relax}. 
% \begin{proposition}
% The optimal solution $\Lambda \in\mbR^{N\times d}$ of the relaxed dual problem \eqref{cvx_nn:d_relax} gives an approximation of $\nabla\log\rho(x)$ evaluated on $x_1,\dots,x_n$. {\color{red} $x_N$? approximation in what function space? or represented by what basis? is it $-\nabla \log \rho$?} %below:
% % \begin{equation}
% %     \bmbm{\nabla \log\rho(x_1)^T\\\vdots\\\nabla \log\rho(x_N)^T} \approx -\Lambda.
% % \end{equation}
% % \begin{equation}
% %     \bmbm{(\nabla \log\rho(x_1)-\nabla\log\pi(x_1))^T\\\vdots\\(\nabla \log\rho(x_N)-\nabla\log\pi(x_N))^T} \approx -\Lambda-Y.
% % \end{equation}
% \end{proposition}
Note that at $(Z^*,W^*,\alpha^*)$ we obtain the optimal approximation of $\nabla\log\rho-\nabla\log\pi$ at $x_1,\dots,x_N$ in the family of two-layer squared-ReLU networks \eqref{eq:twolayerNN}. 
% Compared to $Z^*$, $\tilde Z^*$ achieves 
Smaller or equal objective value of the relaxed bi-dual problem \eqref{cvx_nn:bidual} can be achieved at $(\tilde Z^*,\{\tilde S^{(j,+)},\tilde S^{(j,-)}\}_{j=1}^p)$ than at $(Z^*,\{S^{(j,+)},S^{(j,-)}\}_{j=1}^p)$. Therefore, we can view $\tilde Z^*$ gives an optimal approximation of $\nabla\log\rho-\nabla\log\pi$ evaluated on $x_1,\dots,x_N$ in a broader function family including the two-layer squared ReLU neural networks. 

From the derivation of the relaxed bi-dual problem, we have the relation $\tilde Z^*=-\Lambda^*-Y$, where $(\Lambda^*,\{r^{(j,+)},r^{(j,-)})$ is optimal to the relaxed dual problem \eqref{cvx_nn:d_relax} and $(\tilde Z^*,\{\tilde S^{(j,+)},\tilde S^{(j,-)}\}_{j=1}^p)$ is optimal to the relaxed bi-dual problem \eqref{cvx_nn:bidual}. Therefore, by solving $\Lambda^*$ from the relaxed dual problem \eqref{cvx_nn:d_relax},  we can use $-\Lambda^*-Y$ as the approximation of $\nabla\log\rho-\nabla\log\pi$ evaluated on $x_1,\dots,x_N$.

\begin{remark}
We note that solving the proposed convex optimization problem \ref{cvx_nn:d_relax} renders the approximation of the Wasserstein gradient direction. Compared to the two-layer ReLU networks, it induces a broader class of functions represented by $\{S^{(j,+)},S^{(j,-)}\}_{j=1}^p$. This contains more variables than the neural network function.
\end{remark}

% \begin{proof}
% The variable $Z\in\mbR^{N\times d}$ in the primal problem \eqref{equ:wgd_reg_rep} and the relaxed bidual problem \eqref{cvx_nn:bidual} represent $\nabla\Phi(x)$ evaluated on samples $x_1,x_2,\dots,x_n$. {\color{red} $x_N?$} The optimal solution $Z^*$ to the relaxed bidual problem \eqref{cvx_nn:bidual} corresponds to the optimal approximation of  $\nabla\log\rho(x)-\nabla\log\pi(x)$ evaluated on $x_1,\dots,x_N$. {\color{red} why optimal? in what approximation space? do they have the same value?} The relaxed bidual problem \eqref{cvx_nn:bidual} can be viewed as approximating the Wasserstein gradient direction $\nabla \log \rho-\nabla \log\pi$ in a wider function family including the gradient of two-layer squared ReLU neural networks. {\color{red} can you define or characterize the ``wider function family"?} As \eqref{cvx_nn:d_relax} is a convex problem, the optimal values of \eqref{cvx_nn:d_relax} and \eqref{cvx_nn:bidual} are same. From the KKT condition, {\color{red} KKT of what? can you add it here or put it in the appendix?} we have
% \begin{equation*}
%     Z^*=-\Lambda^*-Y,
% \end{equation*}
% where $\Lambda^*$ is the optimal solution to \eqref{cvx_nn:d_relax} and $Z^*$ is the optimal solution to \eqref{cvx_nn:bidual}. Recall that $Y=\bmbm{\nabla\log \pi(x_1)^T\\\vdots\\\nabla \log \pi(x_N)^T }$. This completes the proof. {\color{red} unless $Z^*$ equals...}
% \end{proof}

\subsection{Practical implementation}
Although the number $p$ of all possible hyper-plane arrangements is upper bounded by $2r((N-1)e/r)^r$ with $r=\rank(X)$, %{\color{red} why is this correct? do you assume rank($X$) = d and use the result in Cover 1965? If so, state it here. also, isn't it $N-1$ instead of N?}
it is computationally costly to enumerate all possible $p$ matrices $D_1,\dots,D_p$ to represent the constraints in the relaxed dual problem \eqref{prop:dual_relax}. %{\color{red} do you enumerate the matrices $D_1,\dots,D_p$, or the constraints in (16)?} 
In practice, %{\color{red} can you add some reference for this subsampling method?} 
we first randomly sample $M$ i.i.d. random vectors $u_1,\dots,u_M\sim \mcN(0,I_d)$ and generate a subset $\hat \mcS$ of $\mcS$ as follows:
\begin{equation}
    \hat \mcS=\{\diag(\mbI(Xu_j\geq 0)|j\in[M]\}. 
\end{equation}
Then, we optimize the randomly sub-sampled version  of the relaxed dual problem based on the subset $\hat \mcS$ and obtain the solution $\Lambda$. We then use $-\Lambda-Y$ as the direction to update the particle system $X$. 

If the regularization parameter is too large, then we will have $-\Lambda-Y=0$, which makes the particle system unchanged. Therefore, to ensure that $\tilde \beta$ is not too large, we decay $\tilde \beta$ by a factor $\gamma_1\in(0,1)$. This also appears in \citep{ergen2021demystifying}. On the other hand, if $\tilde \beta$ is too small resulting the relaxed dual problem \eqref{prop:dual_relax} infeasible, we increase $\tilde \beta$ by multiplying $\gamma_2^{-1}$, where $\gamma_2\in(0,1)$. 
Detailed explanation of the adjustment of the regularization parameter can be found in Appendix \ref{app:reg}. The overall algorithm is summarized in Algorithm \ref{alg:cnwd}.
% {\color{red} what are $u_j$ here?}
% Denote $\hat p=|\hat \mcS|$. Then, we consider the following randomly sub-sampled version  of the relaxed dual problem {\color{red} what is the difference beyond the subsampling $[p]$ and $[\hat{p}]$? if not, I think there is no need to repeat the optimization problem here.}
% \begin{equation}\label{cvx_nn:d_relax_rng}
% \begin{aligned}
% \max\;& -\frac{1}{2}\|\Lambda+Y\|_2^2,\\
% \text{ s.t. }
% &\hat A_j(\Lambda)+\hat B_j+\sum_{n=0}^N r_n^{(j,-)} \hat H^{(j)}_n+\tilde \beta e_{d+1}e_{d+1}^T\succeq 0,\\
% &-\hat A_j(\Lambda)-\hat B_j+\sum_{n=0}^N r_n^{(j,+)} \hat H^{(j)}_n+\tilde \beta e_{d+1}e_{d+1}^T\succeq 0,\\
% &r^{(j,-)}\geq 0, r^{(j,+)}\geq 0, j\in[\hat p].
% \end{aligned}
% \end{equation}
% Here we have 
% \begin{equation*}
% \begin{aligned}
%     \tilde A_j(\Lambda) =& \bmbm{-\Lambda^T\hat D_jX-X^T\hat D_j\Lambda&0\\0&0},\\
%     \tilde B_j =& \bmbm{2\tr(\hat D_j)I&0\\0&0}\in \mbR^{(d+1)\times(d+1)},\\
%     \hat H_0^{(j)}=&\bmbm{I_d&0\\0&-1},\\
% \hat H^{(j)}_n=&\bmbm{0&(1-2(\hat D_j)_{nn})x_n\\(1-2(\hat D_j)_{nn})x_n^T&0}, n\in[N].
% \end{aligned}
% \end{equation*}

% The algorithm based on the relaxed dual problem is summarized in Algorithm \ref{alg:cnwd}. {\color{red} add some description of the algorithm, including the regularization parameter and its update, what are the $\lambda_l$ and $y_l$, stopping criterion (why $L$ instead of some tolerance in terms of convergence? what is the convergence criterion?)}
\begin{algorithm}[!htp]
\caption{Convex neural Wasserstein descent}
\label{alg:cnwd}
\begin{algorithmic}[1]
\REQUIRE initial positions $\{x_0^n\}_{n=1}^N$, step size $\alpha_l$, %number of iteration $L$, 
initial regularization parameter $\tilde \beta_0$, $\gamma_1,\gamma_2\in(0,1)$.
% \FOR{$l=0,1,\dots L$}
\WHILE{not converge}
\STATE Form $X_l$ and $Y_l$ based on $\{x_l^n\}_{n=1}^N$ and $\{\nabla \log\pi(x_l^n)\}_{n=1}^N$.
\STATE Solve $\Lambda_l$ from the relaxed dual problem \eqref{cvx_nn:d_relax} with $\tilde \beta=\tilde \beta_l$. 
\IF {the relaxed dual problem with $\tilde \beta=\tilde \beta_l$ is infeasible}
\STATE Set $X_{l+1}=X_l$ for $n\in[N]$ and set $\tilde \beta_{l+1}=\gamma_2^{-1}\tilde \beta_l$.
\ELSE
\STATE Update $X_{l+1}=X_l+\alpha_l(\Lambda_l+Y_l)$ for $n\in[N]$ and set $\tilde \beta_{l+1}=\gamma_1\tilde \beta_l$.
\ENDIF
% \ENDFOR
\ENDWHILE
\end{algorithmic}
\end{algorithm}

We note that the randomly subsampled version of the relaxed dual problem \eqref{cvx_nn:d_relax} involves $2N\hat p$ inequality constraints and $2\hat p$ linear matrix inequality constraints with size $(d+1)\times(d+1)$. For high-dimensional problems, i.e., $d$ is large, the computational cost of solving \eqref{cvx_nn:d_relax} can be large. In this case, we apply the dimension-reduction techniques \citep{zahm2018certified, chen2020projected, wang2021projected} to reduce the parameter dimension $d$ to a data-informed intrinsic dimension $\hat{d}$, which is often very low, i.e., $\hat{d} \ll d$. 
% This can makes the implementation more efficient.

% \subsection{Exact formulation in 1D}
% For $d=1$, we note that $\mcS$ defined in \eqref{equ:s} can be characterized as
% \begin{equation}
%     \mcS=\{\diag(\mbI(X\geq 0)), \diag(\mbI(-X\geq 0))\}.
% \end{equation}
% This implies that $p=|\mcS|=2$. Let $D_1=\diag(\mbI(X\geq 0))$ and $D_2=\diag(\mbI(-X\geq 0))$. We can simply the dual constraints by
% \begin{equation}
%     \max_{j\in[2]}\max_{w:\|w\|_2\leq 1, (-1)^{j-1}w\geq 0} \left|\tr(D_j)\|w\|_2^2-2w^T\Lambda^TD_jXw\right|\leq \tilde \beta,
% \end{equation}
% which is equivalent to
% \begin{equation}
%     \max_{j\in[2]}|\Lambda^TD_jX-\tr(D_j)|\leq \tilde \beta.
% \end{equation}
% Therefore, the dual problem is equivalent to
% \begin{equation}\label{cvx_nn:d_1d}
% \begin{aligned}
% \max \; &-\frac{1}{2}\|\Lambda+Y\|_F^2,\\
% \text{ s.t. }&\max_{j\in[2]}|\Lambda^TD_jX-\tr(D_j)|\leq \tilde \beta,
% \end{aligned}
% \end{equation}
% in variable $\Lambda\in \mbR^{N\times 1}$. 

\section{Numerical experiments}
In this section, we present numerical results to compare %WGD with KDE, 
WGD approximated by neural networks (WGD-NN) and WGD approximated using convex optimization formulation of neural networks (WGD-cvxNN). The performance of the two methods is assessed by the sample goodness-of-fit of the posterior. For WGD-NN, in each iteration, it updates the particle system using \eqref{upd:wgd} with a function $\Phi$ represented by a two-layer squared ReLU neural network. The parameters of the neural network is obtained by directly solving the nonconvex optimization problem \eqref{equ:wgd_reg}. 
We note that it takes longer time by WGD-cvxNN (compared to WGD-NN and WGD) to solve the convex optimization problem. However, this optimization time is often dominated by the time in likelihood evaluation if the model is expensive to solve. %Therefore, the run time is not reported as it is problem dependent.
Moreover, the induced SDPs have specific structures of many similar constraints, whose solution can be accelerated by designing a specialized convex optimization solver. This is left for future work.

\subsection{A toy example}
We test the performance of WGD on a bimodal 2-dimensional double-banana posterior distribution introduced in \citep{detommaso2018stein}. 
% The prior is a standard Gaussian distribution and the forward operator writes
% $$
% \mcF(x)=\log((1-x_1)^2+100(x_2-x_1^2)^2).
% $$
% We take a single observation $y=F(x^*)+\xi$, where $x^*$ is the ground truth drawn from the prior and $\xi$ follows $\mcN(0,\sigma^2I)$ with $\sigma=0.3$. 
% {\color{red}provide the density function} 
We first generate $300$ posterior samples by a Stein variational Newton (SVN) method \citep{detommaso2018stein} as the reference, as shown in Figure \ref{fig:db2}. We evaluate the performance of WGD-NN and WGD-cvxNN by calculating the maximum mean discrepancy (MMD) between their samples in each iteration and the reference samples. In the comparison, we use $N=50$ samples and run for $100$ iterations with step sizes $\alpha_l=10^{-3}$. For WGD-cvxNN, we set $\beta=1$, $\gamma_1=0.95$ and $\gamma_2=0.95^{10}$. For WGD-NN, we use $m=200$ neurons and optimize the regularized training problem \eqref{equ:wgd_reg} using all samples with the Adam optimizer \citep{kingma2014adam} with learning rate $10^{-3}$ for $200$ sub-iterations. %{\color{red} any batch size? or use all samples?} 
We also set the regularization parameter $\beta=1$ and decrease it by a factor of $0.95$ in each iteration. We find that this setup of parameters is more suitable. 

The posterior density and the sample distributions by WGD-cvxNN and WGD-NN at the final step of 100 iterations are shown in Figure \ref{fig:db2}. It can be observed that WGD-cvxNN provides more representative samples than WGD-NN for the posterior density.
%In practice, we note that with a fixed regularization parameter $\beta$, the movement of the particles will quickly go to $0$. This implies that we shall gradually decrease $\beta$ along the path. Otherwise, we will make no progress for further iterations. This is analogous to quenching in some sense.

% \begin{figure}[H]
% \centering
% \begin{minipage}[t]{0.45\textwidth}
% \centering
% \includegraphics[width=\linewidth]{kWNF/toy2d_no1_MMD_cvx_NN.png}
% \end{minipage}
% \caption{MMD as a function of iteration number. Toy two-mode distribution.}
% \end{figure}

% \begin{figure}[H]
% \centering
% \begin{minipage}[t]{0.45\textwidth}
% \centering
% \includegraphics[width=\linewidth]{kWNF/toy2d_no1_MMD_final.png}
% \end{minipage}
% \caption{Samples at the final iteration. Toy two-mode distribution.}
% \end{figure}

% Then, we test on a 2d double banana distribution. 
\begin{figure}[H]
\centering
\begin{minipage}[t]{0.5\textwidth}
\centering
\includegraphics[width=\linewidth]{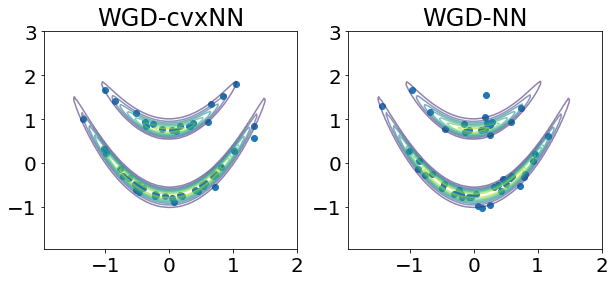}
\end{minipage}
\centering
\begin{minipage}[t]{0.27\textwidth}
\centering
\includegraphics[width=\linewidth]{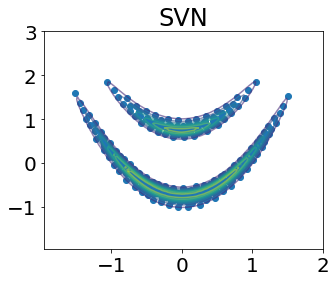}
\end{minipage}
\caption{Posterior density and sample distributions by WGD-cvxNN and WGD-NN at the final step of 100 iterations, compared to the reference SVN samples (right).}\label{fig:db2}
\end{figure}

In Figure \ref{fig:db1}, we plot the MMD of the samples by WGD-cvxNN and WGD-NN compared to the reference SVN samples at each iteration. We observe that the samples by WGD-cvxNN achieves much smaller MMD than those of WGD-NN compared to the reference SVN samples, which is consistent with the results shown in Figure \ref{fig:db2}. For WGD-cvxNN, it takes 572s in total, while for WGD-NN, it takes 16s in total. WGD-cvxNN takes much longer time than WGD-NN as WGD-cvxNN aims to solve for the global minimum of the relaxed convex dual problem.
% The samples generated by WGD-cvxNN also fits  the posterior distribution better.  % Langevin MCMC. We observe that  

\begin{figure}[H]
\centering
\begin{minipage}[t]{0.6\textwidth}
\centering
\includegraphics[width=\linewidth]{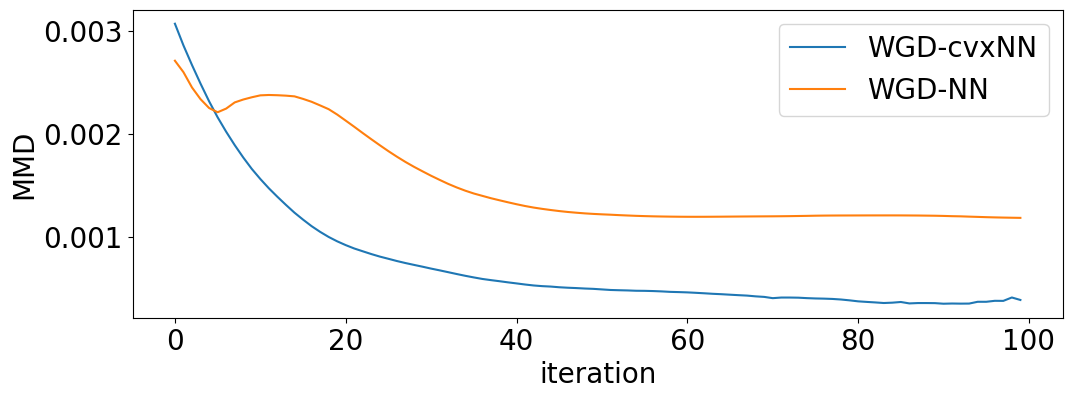}
\end{minipage}
\caption{MMD of WGD-cvxNN and WGD-NN samples compared to the reference SVN samples.}\label{fig:db1}\end{figure}

% \begin{figure}[H]
% \centering
% \begin{minipage}[t]{\textwidth}
% \centering
% \includegraphics[width=\linewidth]{kWNF/NN_toy2d_scalar}
% \caption{Toy 2d example. Scalar output.}\label{fig:NN_toy2d}
% \end{minipage}
% \end{figure}

% \begin{figure}[H]
% \centering
% \begin{minipage}[t]{\textwidth}
% \centering
% \includegraphics[width=\linewidth]{kWNF/NN_toy2d}
% \caption{Toy 2d example. Vector output.}\label{fig:NN_toy2d}
% \end{minipage}
% \end{figure}

\subsection{PDE-constrained nonlinear Bayesian inference}

In this experiment, we consider a nonlinear Bayesian inference problem constrained by the following partial differential equation (PDE) with application to subsurface (Darcy) flow in a physical domain $D = (0, 1)^2$, 
\begin{equation}\label{eq:nonlinear-pde}
\begin{split}
    \mathbf{v} + e^x \nabla u & = 0 \quad \text{ in } D, \\
    \nabla \cdot \mathbf{v} & = h \quad \text{ in } D, \\
\end{split}
\end{equation}
% where the parameter $x$ to the observation on the pressure $u$ is nonlinear. Here 
where $u$ is pressure, $\mathbf{v}$ is velocity, $h$ is force, $e^x$ is a random (permeability) field equipped with a Gaussian prior ${x} \sim \mathcal{N}(x_0, C)$ with covariance operator $C = (-\delta \Delta + \gamma I)^{-\alpha}$ where we set $\delta = 0.1, \gamma = 1, \alpha = 2$ and $x_0 = 0$. We impose Dirichlet boundary conditions $u = 1$ on the top boundary and $u = 0$ on the bottom boundary, and homogeneous Neumann boundary conditions on the left and right boundaries for $u$. We use a finite element method with piecewise linear elements for the discretization of the problem, resulting in 81 dimensions for the discrete parameter. The data is generated as pointwise observation of the pressure field at 49 points equidistantly distributed in $(0, 1)^2$, corrupted with additive $5\%$ Gaussian noise. We use a DILI-MCMC algorithm \cite{cui2016dimension} with 10000 effective samples to compute the sample mean and sample variance, which are used as the reference values to assess the goodness of the samples by pWGD-cvxNN and pWGD-NN.

\begin{figure}[!htb]
\centering
\begin{minipage}[t]{0.4\textwidth}
\centering
\includegraphics[width=\linewidth]{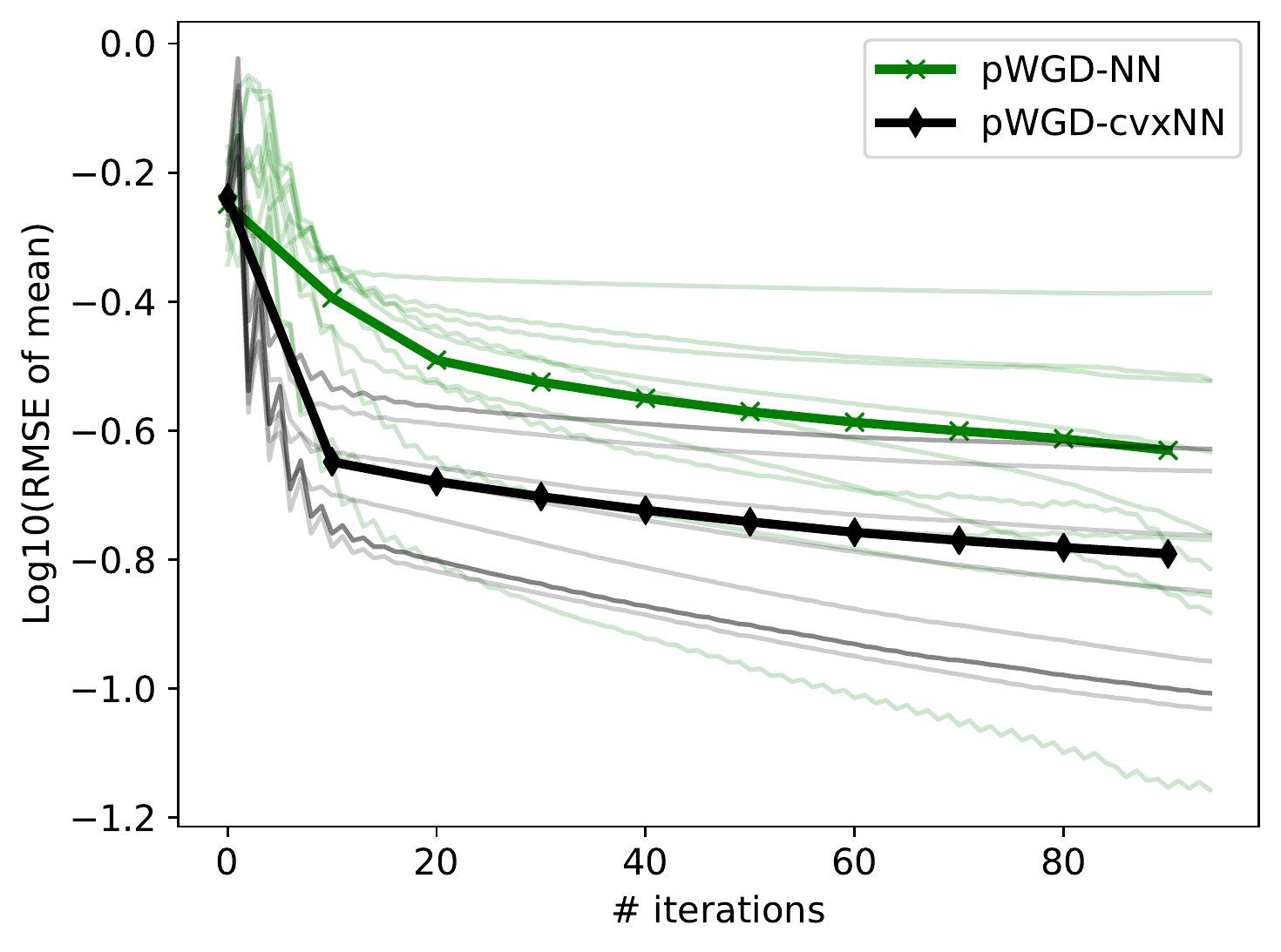}
\end{minipage}
\begin{minipage}[t]{0.4\textwidth}
\centering
\includegraphics[width=\linewidth]{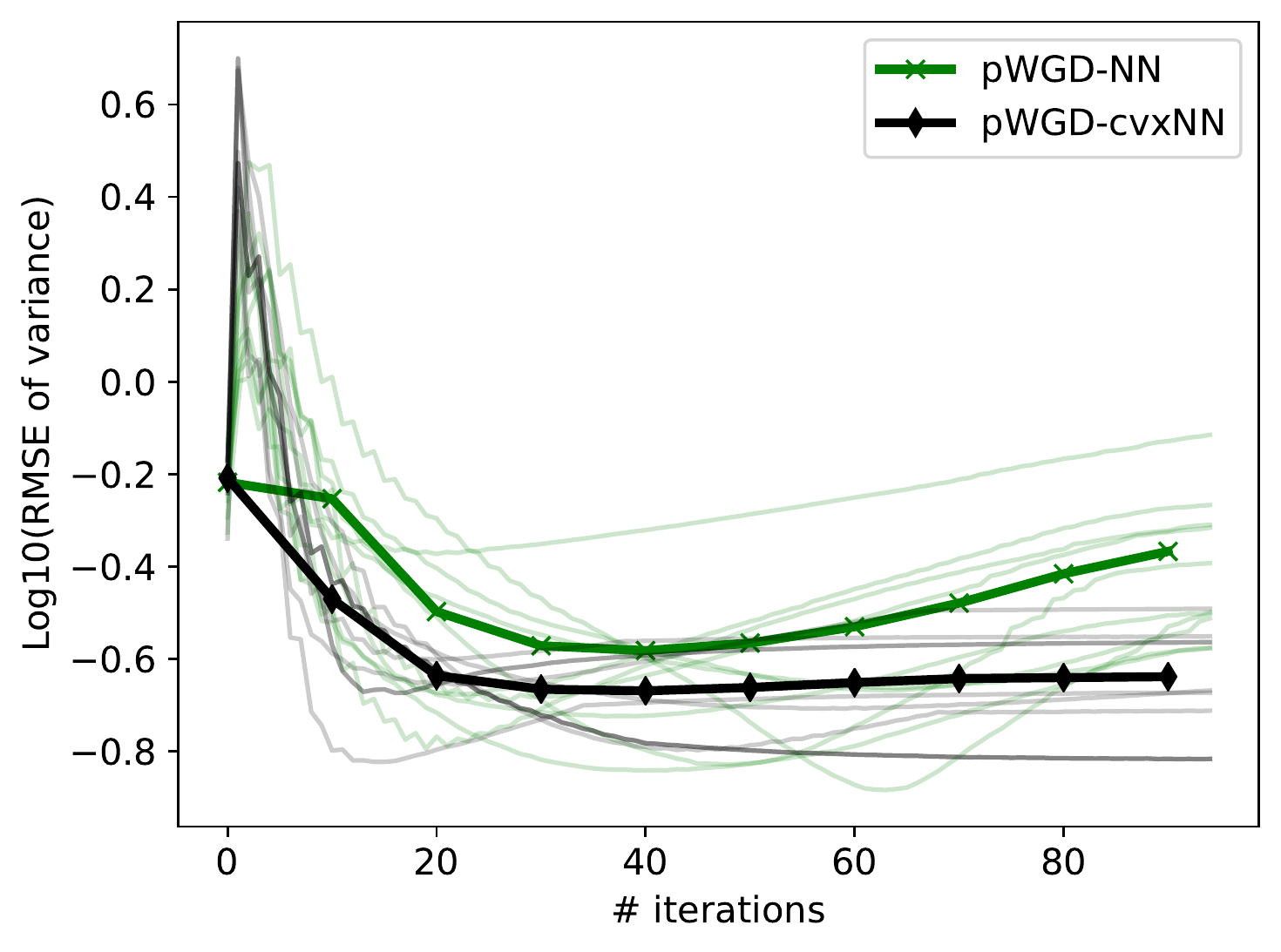}
\end{minipage}
\caption{Ten trials and the RMSE of the sample mean (top) and sample variance (bottom) by pWGD-NN and pWGD-cvxNN at different iterations. Nonlinear inference problem.}\label{fig:nonlinear_PDE}
\end{figure}

We run pWGD-cvxNN and pWGD-NN with 64 samples for ten trials with step size $\alpha_l = 10^{-3}$, where we set $\beta = 10$, $\gamma_1 = 0.95$, and $\gamma_2 = 0.95^{10}$ for both methods. The RMSE of the sample mean and sample variance are shown in Figure \ref{fig:nonlinear_PDE} for the two methods at each of the iterations. We can observe that pWGD-cvxNN achieves smaller errors for both the sample mean and the sample variance compared to pWGD-NN at each iteration. Moreover, pWGD-cvxNN provides much smaller variation of the sample mean and sample variance for the ten trials compared to pWGD-NN. 

\subsection{Bayesian inference for COVID-19}

In this experiment, we use Bayesian inference to learn the dynamics of the transmission and severity of COVID-19 from the recorded data for New York state, as studied in \cite{chen2020projected}. We use the model, parameter, and data as in \cite{chen2020projected}. More specifically, we use a compartmental model for the modeling of the transmission and outcome of COVID-19. We take the number of hospitalized cases as the observation data to infer a social distancing parameter, a time-dependent stochastic process that is equipped with a Tanh--Gaussian prior to model the transmission reduction effect of social distancing, which becomes 96 dimensions after discretization.

We run a projected Stein variational gradient descent (pSVGD) method \cite{chen2020projected} as the reference, and run pWGD-cvxNN and pWGD-NN using 64 samples for 100 iterations with step size $\alpha_l = 10^{-3}$, where we set $\beta = 10$, $\gamma_1 = 0.95$, and $\gamma_2 = 0.95^{10}$ for both methods as in the last example. From Figure \ref{fig:covid19} we can observe that pWGD-cvxNN produces more consistent results with pSVGD than pWGD-NN for both the sample mean and 90\% credible interval, both in the inference of the social distancing parameter and in the prediction of the hospitalized cases. % using the parameter samples.

\begin{figure}[!htb]
 \centering
 \begin{minipage}[t]{0.4\textwidth}
\centering
\includegraphics[trim=20 0 20 20, clip, width=\linewidth]{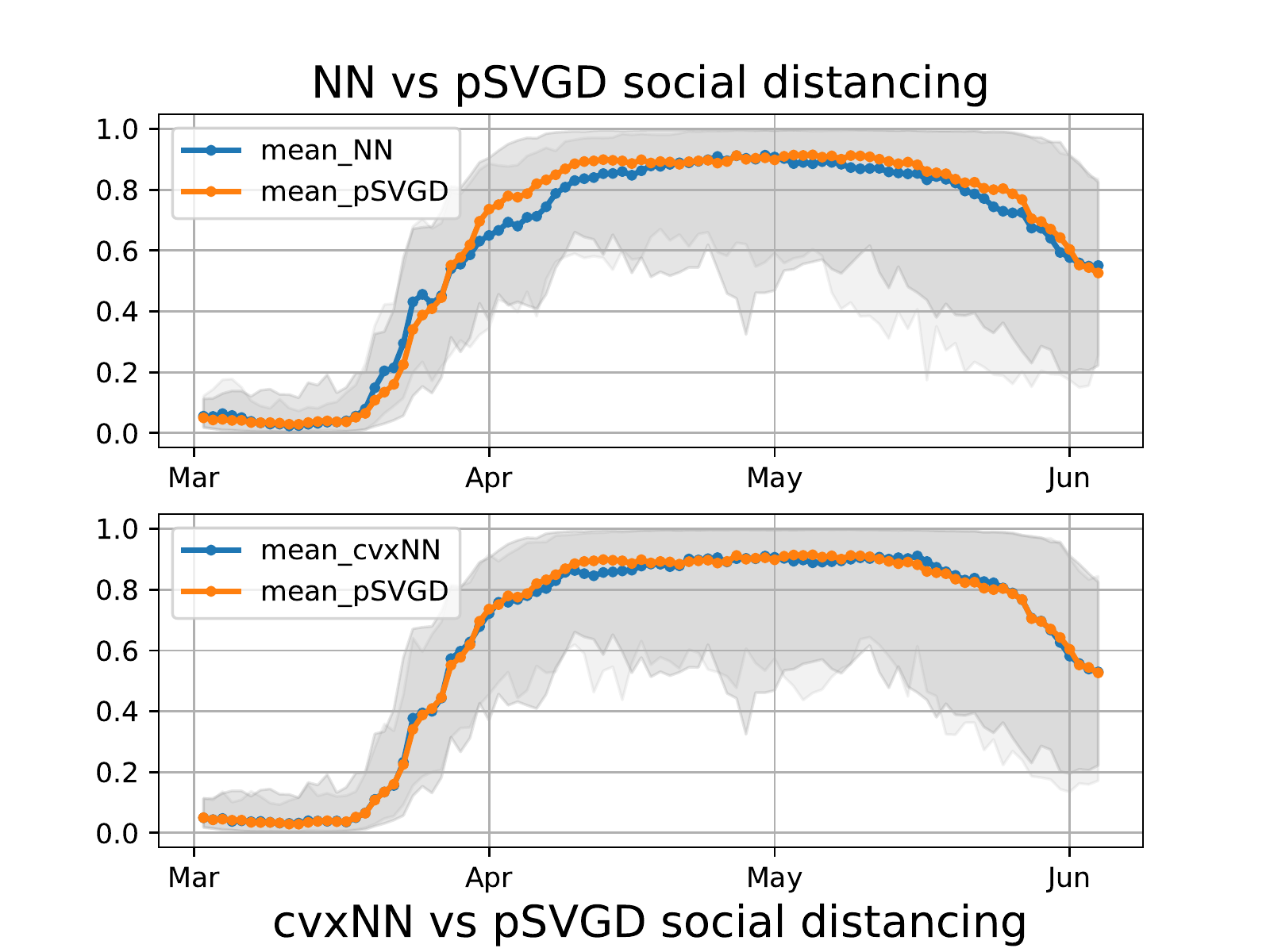}
 \end{minipage}
 \begin{minipage}[t]{0.4\textwidth}
\includegraphics[trim=20 0 20 20, clip, width=\linewidth]{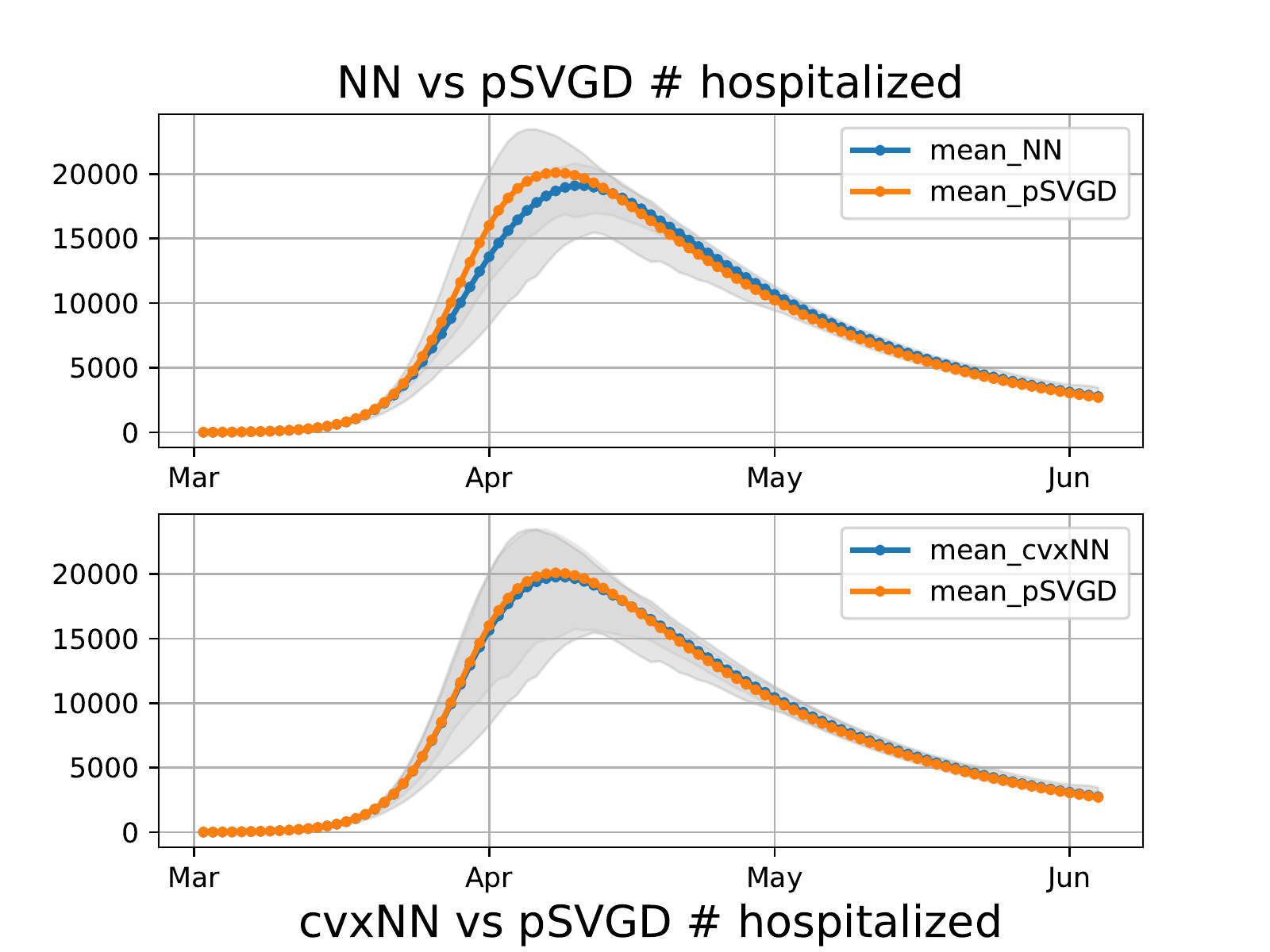}
 \end{minipage}
\caption{Comparison of pWGD-cvxNN and pWGD-NN to the reference by pSVGD for Bayesian inference of the social distancing parameter (top) from the data of the hospitalized cases (bottom) with sample mean and 90\% credible interval.}\label{fig:covid19}
\end{figure}

\section{Conclusion}
In the context of variational Wasserstein gradient descent methods for Bayesian inference, we consider the approximation of the Wasserstein gradient direction by the gradient of functions in the family of two-layer neural networks. We propose a convex SDP relaxation of the dual of the variational primal problem, which can be solved efficiently using convex optimization methods instead of directly training the neural network as a nonconvex optimization problem. In particular, we established that the gradient obtained by the new formulation and convex optimization is at least as good as the optimal approximation of the Wasserstein gradient direction by functions in the family of two-layer neural networks, which is demonstrated by various numerical experiments. In future works, we expect to extend our convex neural network approximations to generalized Wasserstein flows.
% Numerical comparisons with the neural network approximation of Wasserstein gradient direction illustrate the strength of approximating the Wasserstein gradient using the convex SDP relaxation.

%%%%%%%%%%%%%%%%%%%%%%%%%%%%%%%%%%%%%%%%%%%%%%%%%%%%%%%%%%%%

\bibliography{WIG}

\begin{thebibliography}{41}
\providecommand{\natexlab}[1]{#1}
\providecommand{\url}[1]{\texttt{#1}}
\expandafter\ifx\csname urlstyle\endcsname\relax
  \providecommand{\doi}[1]{doi: #1}\else
  \providecommand{\doi}{doi: \begingroup \urlstyle{rm}\Url}\fi

\bibitem[Alvarez-Melis et~al.(2021)Alvarez-Melis, Schiff, and
  Mroueh]{alvarez2021optimizing}
Alvarez-Melis, D., Schiff, Y., and Mroueh, Y.
\newblock Optimizing functionals on the space of probabilities with input
  convex neural networks.
\newblock \emph{arXiv preprint arXiv:2106.00774}, 2021.

\bibitem[Ambrosio et~al.(2005)Ambrosio, Gigli, and
  Savar{\'e}]{ambrosio2005gradient}
Ambrosio, L., Gigli, N., and Savar{\'e}, G.
\newblock \emph{Gradient flows: in metric spaces and in the space of
  probability measures}.
\newblock Springer Science \& Business Media, 2005.

\bibitem[Bartan \& Pilanci(2021)Bartan and Pilanci]{bartan2021neural}
Bartan, B. and Pilanci, M.
\newblock Neural spectrahedra and semidefinite lifts: Global convex
  optimization of polynomial activation neural networks in fully
  polynomial-time.
\newblock \emph{arXiv preprint arXiv:2101.02429}, 2021.

\bibitem[Bonet et~al.(2021)Bonet, Courty, Septier, and
  Drumetz]{bonet2021sliced}
Bonet, C., Courty, N., Septier, F., and Drumetz, L.
\newblock Sliced-wasserstein gradient flows.
\newblock \emph{arXiv preprint arXiv:2110.10972}, 2021.

\bibitem[Bunne et~al.(2021)Bunne, Meng-Papaxanthos, Krause, and
  Cuturi]{bunne2021jkonet}
Bunne, C., Meng-Papaxanthos, L., Krause, A., and Cuturi, M.
\newblock Jkonet: Proximal optimal transport modeling of population dynamics.
\newblock \emph{arXiv preprint arXiv:2106.06345}, 2021.

\bibitem[Carrillo et~al.(2021{\natexlab{a}})Carrillo, Craig, Wang, and
  Wei]{carrillo2021primal}
Carrillo, J.~A., Craig, K., Wang, L., and Wei, C.
\newblock Primal dual methods for wasserstein gradient flows.
\newblock \emph{Foundations of Computational Mathematics}, pp.\  1--55,
  2021{\natexlab{a}}.

\bibitem[Carrillo et~al.(2021{\natexlab{b}})Carrillo, Matthes, and
  Wolfram]{carrillo2021lagrangian}
Carrillo, J.~A., Matthes, D., and Wolfram, M.-T.
\newblock Lagrangian schemes for wasserstein gradient flows.
\newblock \emph{Handbook of Numerical Analysis}, 22:\penalty0 271--311,
  2021{\natexlab{b}}.

\bibitem[Chen \& Ghattas(2020)Chen and Ghattas]{chen2020projected}
Chen, P. and Ghattas, O.
\newblock Projected stein variational gradient descent.
\newblock \emph{Advances in Neural Information Processing Systems},
  33:\penalty0 1947--1958, 2020.

\bibitem[Cover(1965)]{cover1965geometrical}
Cover, T.~M.
\newblock Geometrical and statistical properties of systems of linear
  inequalities with applications in pattern recognition.
\newblock \emph{IEEE transactions on electronic computers}, \penalty0
  (3):\penalty0 326--334, 1965.

\bibitem[Cui et~al.(2016)Cui, Law, and Marzouk]{cui2016dimension}
Cui, T., Law, K.~J., and Marzouk, Y.~M.
\newblock Dimension-independent likelihood-informed mcmc.
\newblock \emph{Journal of Computational Physics}, 304:\penalty0 109--137,
  2016.

\bibitem[Detommaso et~al.(2018)Detommaso, Cui, Spantini, Marzouk, and
  Scheichl]{detommaso2018stein}
Detommaso, G., Cui, T., Spantini, A., Marzouk, Y., and Scheichl, R.
\newblock A stein variational newton method.
\newblock \emph{arXiv preprint arXiv:1806.03085}, 2018.

\bibitem[di~Langosco et~al.(2021)di~Langosco, Fortuin, and
  Strathmann]{di2021neural}
di~Langosco, L.~L., Fortuin, V., and Strathmann, H.
\newblock Neural variational gradient descent.
\newblock \emph{arXiv preprint arXiv:2107.10731}, 2021.

\bibitem[Diamond \& Boyd(2016)Diamond and Boyd]{diamond2016cvxpy}
Diamond, S. and Boyd, S.
\newblock {CVXPY}: {A} {P}ython-embedded modeling language for convex
  optimization.
\newblock \emph{Journal of Machine Learning Research}, 17\penalty0
  (83):\penalty0 1--5, 2016.

\bibitem[Ergen et~al.(2021)Ergen, Sahiner, Ozturkler, Pauly, Mardani, and
  Pilanci]{ergen2021demystifying}
Ergen, T., Sahiner, A., Ozturkler, B., Pauly, J., Mardani, M., and Pilanci, M.
\newblock Demystifying batch normalization in relu networks: Equivalent convex
  optimization models and implicit regularization.
\newblock \emph{arXiv preprint arXiv:2103.01499}, 2021.

\bibitem[Fan et~al.(2021)Fan, Taghvaei, and Chen]{fan2021variational}
Fan, J., Taghvaei, A., and Chen, Y.
\newblock Variational wasserstein gradient flow.
\newblock \emph{arXiv preprint arXiv:2112.02424}, 2021.

\bibitem[Feng et~al.(2021)Feng, Gao, Huang, Jiao, and Liu]{feng2021relative}
Feng, X., Gao, Y., Huang, J., Jiao, Y., and Liu, X.
\newblock Relative entropy gradient sampler for unnormalized distributions.
\newblock \emph{arXiv preprint arXiv:2110.02787}, 2021.

\bibitem[Frogner \& Poggio(2020)Frogner and Poggio]{frogner2020approximate}
Frogner, C. and Poggio, T.
\newblock Approximate inference with wasserstein gradient flows.
\newblock In \emph{International Conference on Artificial Intelligence and
  Statistics}, pp.\  2581--2590. PMLR, 2020.

\bibitem[Hornik et~al.(1989)Hornik, Stinchcombe, and
  White]{hornik1989multilayer}
Hornik, K., Stinchcombe, M., and White, H.
\newblock Multilayer feedforward networks are universal approximators.
\newblock \emph{Neural networks}, 2\penalty0 (5):\penalty0 359--366, 1989.

\bibitem[Hwang et~al.(2021)Hwang, Kim, Park, and Son]{hwang2021deep}
Hwang, H.~J., Kim, C., Park, M.~S., and Son, H.
\newblock The deep minimizing movement scheme.
\newblock \emph{arXiv preprint arXiv:2109.14851}, 2021.

\bibitem[Jeyakumar \& Li(2014)Jeyakumar and Li]{jeyakumar2014trust}
Jeyakumar, V. and Li, G.
\newblock Trust-region problems with linear inequality constraints: exact sdp
  relaxation, global optimality and robust optimization.
\newblock \emph{Mathematical Programming}, 147\penalty0 (1):\penalty0 171--206,
  2014.

\bibitem[Jordan et~al.(1998)Jordan, Kinderlehrer, and
  Otto]{jordan1998variational}
Jordan, R., Kinderlehrer, D., and Otto, F.
\newblock The variational formulation of the fokker--planck equation.
\newblock \emph{SIAM journal on mathematical analysis}, 29\penalty0
  (1):\penalty0 1--17, 1998.

\bibitem[Junge et~al.(2017)Junge, Matthes, and Osberger]{junge2017fully}
Junge, O., Matthes, D., and Osberger, H.
\newblock A fully discrete variational scheme for solving nonlinear
  fokker--planck equations in multiple space dimensions.
\newblock \emph{SIAM Journal on Numerical Analysis}, 55\penalty0 (1):\penalty0
  419--443, 2017.

\bibitem[Kingma \& Ba(2014)Kingma and Ba]{kingma2014adam}
Kingma, D.~P. and Ba, J.
\newblock Adam: A method for stochastic optimization.
\newblock \emph{arXiv preprint arXiv:1412.6980}, 2014.

\bibitem[Kruse et~al.(2019)Kruse, Detommaso, Scheichl, and
  K{\"o}the]{kruse2019hint}
Kruse, J., Detommaso, G., Scheichl, R., and K{\"o}the, U.
\newblock Hint: Hierarchical invertible neural transport for density estimation
  and bayesian inference.
\newblock \emph{arXiv preprint arXiv:1905.10687}, 2019.

\bibitem[Lan et~al.(2021)Lan, Li, and Shahbaba]{lan2021scaling}
Lan, S., Li, S., and Shahbaba, B.
\newblock Scaling up bayesian uncertainty quantification for inverse problems
  using deep neural networks.
\newblock \emph{arXiv preprint arXiv:2101.03906}, 2021.

\bibitem[Lin et~al.(2021{\natexlab{a}})Lin, Fung, Li, Nurbekyan, and
  Osher]{lin2021alternating}
Lin, A.~T., Fung, S.~W., Li, W., Nurbekyan, L., and Osher, S.~J.
\newblock Alternating the population and control neural networks to solve
  high-dimensional stochastic mean-field games.
\newblock \emph{Proceedings of the National Academy of Sciences}, 118\penalty0
  (31), 2021{\natexlab{a}}.

\bibitem[Lin et~al.(2021{\natexlab{b}})Lin, Li, Osher, and
  Mont{\'u}far]{lin2021wasserstein}
Lin, A.~T., Li, W., Osher, S., and Mont{\'u}far, G.
\newblock Wasserstein proximal of gans.
\newblock \emph{arXiv preprint arXiv:2102.06862}, 2021{\natexlab{b}}.

\bibitem[Liu et~al.(2019)Liu, Zhuo, Cheng, Zhang, and
  Zhu]{liu2019understanding}
Liu, C., Zhuo, J., Cheng, P., Zhang, R., and Zhu, J.
\newblock Understanding and accelerating particle-based variational inference.
\newblock In \emph{International Conference on Machine Learning}, pp.\
  4082--4092. PMLR, 2019.

\bibitem[Liu \& Wang(2016)Liu and Wang]{SVGD}
Liu, Q. and Wang, D.
\newblock Stein variational gradient descent: A general purpose bayesian
  inference algorithm.
\newblock In \emph{Advances in neural information processing systems}, pp.\
  2378--2386, 2016.

\bibitem[Liutkus et~al.(2019)Liutkus, Simsekli, Majewski, Durmus, and
  St{\"o}ter]{liutkus2019sliced}
Liutkus, A., Simsekli, U., Majewski, S., Durmus, A., and St{\"o}ter, F.-R.
\newblock Sliced-wasserstein flows: Nonparametric generative modeling via
  optimal transport and diffusions.
\newblock In \emph{International Conference on Machine Learning}, pp.\
  4104--4113. PMLR, 2019.

\bibitem[Lu et~al.(2017)Lu, Pu, Wang, Hu, and Wang]{lu2017expressive}
Lu, Z., Pu, H., Wang, F., Hu, Z., and Wang, L.
\newblock The expressive power of neural networks: A view from the width.
\newblock In \emph{Proceedings of the 31st International Conference on Neural
  Information Processing Systems}, pp.\  6232--6240, 2017.

\bibitem[Mokrov et~al.(2021)Mokrov, Korotin, Li, Genevay, Solomon, and
  Burnaev]{mokrov2021large}
Mokrov, P., Korotin, A., Li, L., Genevay, A., Solomon, J., and Burnaev, E.
\newblock Large-scale wasserstein gradient flows.
\newblock \emph{arXiv preprint arXiv:2106.00736}, 2021.

\bibitem[Onken et~al.(2020)Onken, Fung, Li, and Ruthotto]{onken2020ot}
Onken, D., Fung, S.~W., Li, X., and Ruthotto, L.
\newblock Ot-flow: Fast and accurate continuous normalizing flows via optimal
  transport.
\newblock \emph{arXiv preprint arXiv:2006.00104}, 2020.

\bibitem[Otto(2001)]{tgode}
Otto, F.
\newblock The geometry of dissipative evolution equations: the porous medium
  equation.
\newblock \emph{Communications in Partial Differential Equations}, 26\penalty0
  (1-2):\penalty0 101--174, 2001.

\bibitem[Pilanci \& Ergen(2020)Pilanci and Ergen]{pilanci2020neural}
Pilanci, M. and Ergen, T.
\newblock Neural networks are convex regularizers: Exact polynomial-time convex
  optimization formulations for two-layer networks.
\newblock In \emph{International Conference on Machine Learning}, pp.\
  7695--7705. PMLR, 2020.

\bibitem[Rezende \& Mohamed(2015)Rezende and Mohamed]{rezende2015variational}
Rezende, D. and Mohamed, S.
\newblock Variational inference with normalizing flows.
\newblock In \emph{International conference on machine learning}, pp.\
  1530--1538. PMLR, 2015.

\bibitem[Sahiner et~al.(2020)Sahiner, Ergen, Pauly, and
  Pilanci]{sahiner2020vector}
Sahiner, A., Ergen, T., Pauly, J., and Pilanci, M.
\newblock Vector-output relu neural network problems are copositive programs:
  Convex analysis of two layer networks and polynomial-time algorithms.
\newblock \emph{arXiv preprint arXiv:2012.13329}, 2020.

\bibitem[Stuart(2010)]{ivabp}
Stuart, A.~M.
\newblock Inverse problems: a {B}ayesian perspective.
\newblock \emph{Acta numerica}, 19:\penalty0 451--559, 2010.

\bibitem[Villani(2003)]{tiot}
Villani, C.
\newblock \emph{Topics in optimal transportation}.
\newblock American Mathematical Soc., 2003.

\bibitem[Wang et~al.(2021)Wang, Chen, and Li]{wang2021projected}
Wang, Y., Chen, P., and Li, W.
\newblock Projected wasserstein gradient descent for high-dimensional bayesian
  inference.
\newblock \emph{arXiv preprint arXiv:2102.06350}, 2021.

\bibitem[Zahm et~al.(2018)Zahm, Cui, Law, Spantini, and
  Marzouk]{zahm2018certified}
Zahm, O., Cui, T., Law, K., Spantini, A., and Marzouk, Y.
\newblock Certified dimension reduction in nonlinear bayesian inverse problems.
\newblock \emph{arXiv preprint arXiv:1807.03712}, 2018.

\end{thebibliography}
\bibliographystyle{icml2022}

\newpage
\appendix
\onecolumn
\section{Codes for numerical experiment}
All codes for the numerical experiment can be found in \url{https://github.com/ai-submit/OptimalWasserstein}.
\section{Additional numerical experiment}
\subsection{PDE-constrained linear Bayesian inference}
In this experiment, we consider a linear Bayesian inference problem constrained by a partial differential equation (PDE) model for contaminant diffusion in environmental engineering in domain $D = (0, 1)$,
$$
- \kappa \Delta u + \nu u = x \quad \text{ in } D,
$$
where $x$ is a contaminant source field parameter in domain $D$, $u$ is the contaminant concentration which we can observe at some locations, $\kappa$ and $\nu$ are diffusion and reaction coefficients. For simplicity, we set $\kappa, \nu = 1$, $u(0)=u(1)=0$, and consider 15 pointwise observations of $u$ with $1\%$ noise, equidistantly distributed in $D$. We consider a Gaussian prior distribution $x \sim \mathcal{N}(0, C)$ with covariance given by a differential operator $C=(-\delta\Delta + \gamma I)^{-\alpha}$ with $\delta, \gamma, \alpha > 0$ representing the correlation length and variance, which is commonly used in geoscience. We set $\delta = 0.1, \gamma = 1, \alpha = 1$. In this linear setting, the posterior is Gaussian with the mean and covariance given analytically, which are used as reference to assess the sample goodness. We solve this forward model by a finite element method with piece-wise linear elements on a uniform mesh of size $2^k$, $k \geq 1$. We project this high-dimensional parameter to the data-informed low dimensions as in \cite{wang2021projected} to alleviate the curse of dimensionality when applying WGD-cvxNN and WGD-NN, which we call pWGD-cvxNN and pWGD-NN, respectively. For $k=4$ we have $17$ dimensions for the discrete parameter and $4$ dimensions after projection. 

We run pWGD-cvxNN and pWGD-NN using 16 samples for 200 iterations with $\alpha_l = 10^{-3}$, $\beta = 5$, $\gamma_1 = 0.95$, and $\gamma_2 = 0.95^{10}$ for both methods. We use $m = 200$ neurons for pWGD-NN and train it by the Adam optimizer for 200 sub-iterations as in the first example. From Figure \ref{fig:linear_PDE}, we observe that pWGD-cvxNN achieves better root mean squared error (RMSE) than pWGD-NN for both the sample mean and the sample variance compared to the reference.

\begin{figure}[H]
\centering
\begin{minipage}[t]{0.45\textwidth}
\centering
\includegraphics[width=\linewidth]{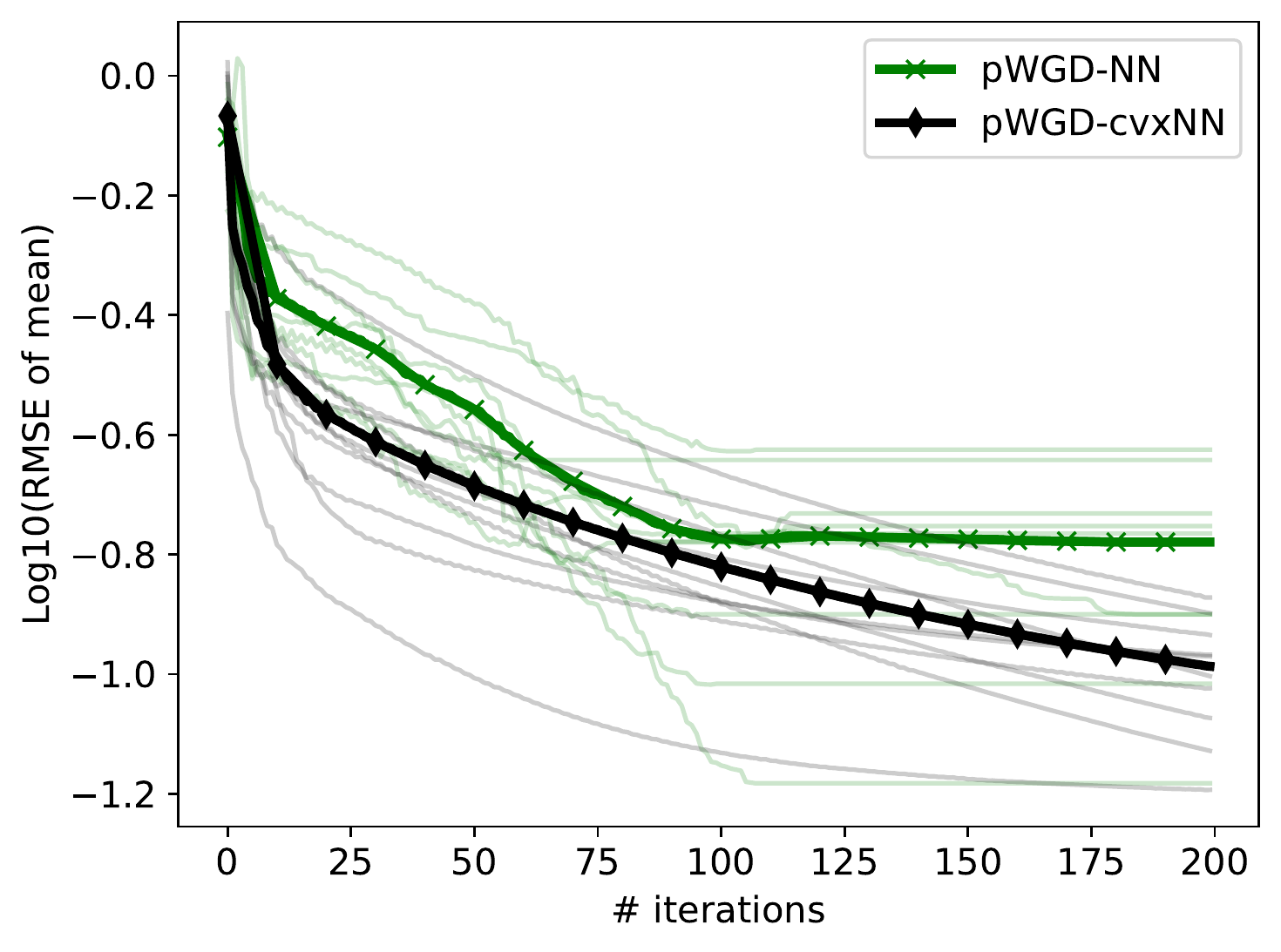}
\end{minipage}
\begin{minipage}[t]{0.45\textwidth}
\centering
\includegraphics[width=\linewidth]{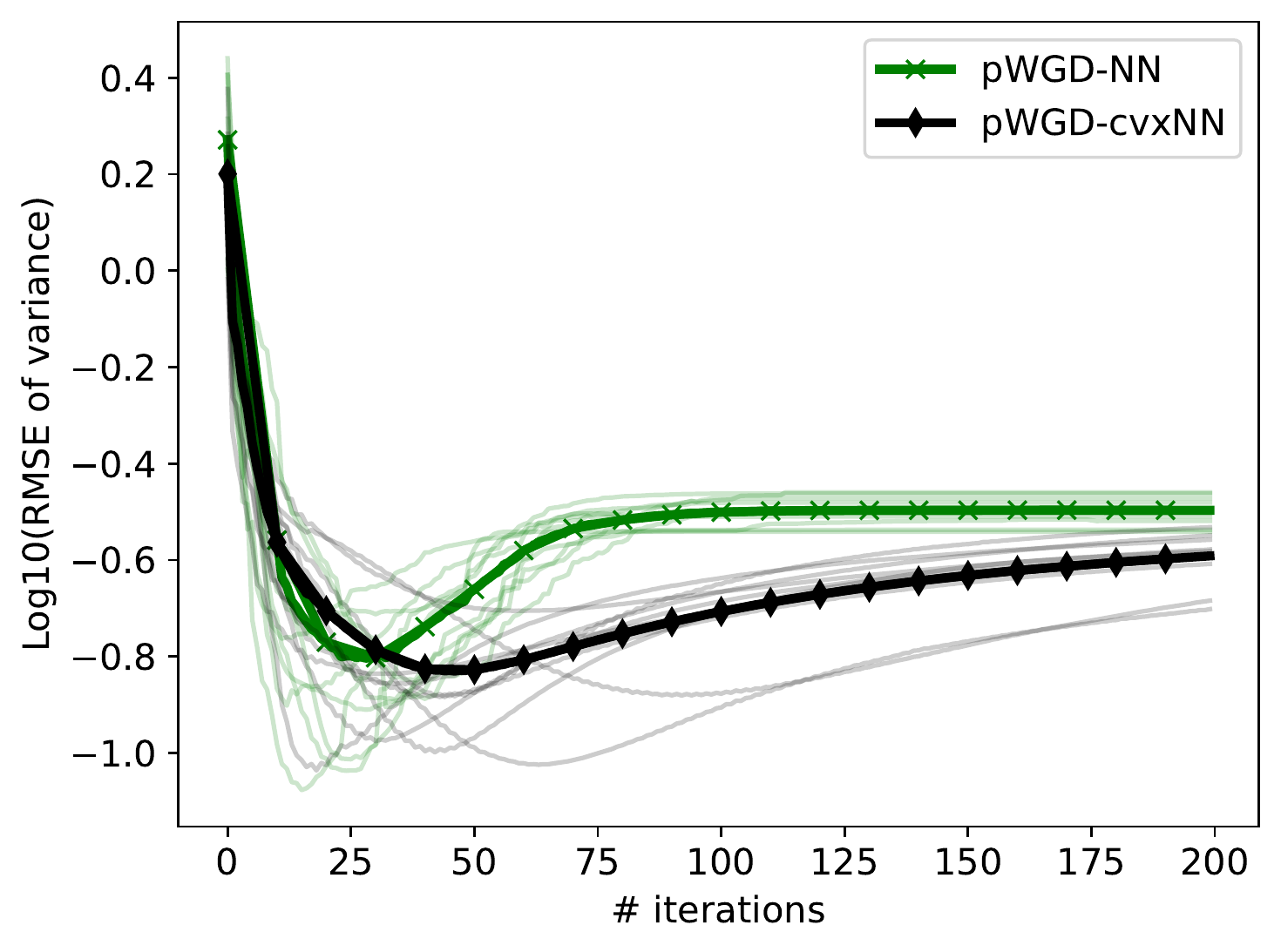}
\end{minipage}
\caption{Ten trials and the RMSE of the sample mean (top) and sample variance (bottom) by pWGD-NN and pWGD-cvxNN at different iterations. Linear inference problem.}\label{fig:linear_PDE}
\end{figure}

\section{Choice of the regularization parameter}\label{app:reg}
%A remaining question is that whether the domain of \eqref{cvx_nn:d_relax} is non-empty. {\color{red} this non-empty or not issue is a surprise here. it should be discussed right after \eqref{cvx_nn:d_relax} with a pointer to here for the solution} Consider the following SDP
As the constraints in the relaxed dual problem \eqref{cvx_nn:d_relax} depends on the regularization parameter $\tilde \beta$, it is possible that for small $\tilde \beta$, the  relaxed dual problem \eqref{cvx_nn:d_relax} is infeasible. Consider the following SDP
\begin{equation}
\begin{aligned}
\min\; \tilde \beta,\text{ s.t. }
&\tilde A_j(\Lambda)+\tilde B_j+\sum_{n=0}^N r_n^{(j,-)} H^{(j)}_n+\tilde \beta e_{d+1}e_{d+1}^T\succeq 0,\\
&-\tilde A_j(\Lambda)-\tilde B_j+\sum_{n=0}^N r_n^{(j,+)} H^{(j)}_n+\tilde \beta e_{d+1}e_{d+1}^T\succeq 0,\\
&r^{(j,-)}\geq 0, r^{(j,+)}\geq 0, j\in[p].
\end{aligned}
\end{equation}
Here the variables are $\tilde \beta, \Lambda $ and $\{r^{(j,+)},r^{(j,-)}\}_{j=1}^p$. 
Let $\tilde \beta_1$ be the optimal value of the above problem. Then, only for $\tilde \beta \geq \tilde \beta_1$, there exists $\Lambda\in \mbR^{N\times d}$ satisfying the constraints in \eqref{cvx_nn:d_relax}. In other words, the relaxed dual problem \eqref{cvx_nn:d_relax} is feasible. %, i.e., . %{\color{red} why this is true? what is the relation between (20) and (16)?}
We also note that 
% \begin{equation}
% \begin{aligned}
% &\tilde \beta_\mathrm{min}\\
% =&\min_\Lambda\max_{\|w\|_2\leq 1}  \left|\sum_{n=1}^N(\|w\|_2^2\mbI(w^Tx_n>0)-2y_n^Tw(x_n^Tw)_+)\right|,
% \end{aligned}
% \end{equation}
$\tilde \beta_1$ only depends on the samples $X$ and it does not depend on the value of $\nabla\log\pi$ evaluated on $x_1,\dots,x_N$. %{\color{red} $X$ are samples, right? isn't it also depending on $\log \pi$ and the neural network?} 
On the other hand, consider the following SDP
\begin{equation}
\begin{aligned}
\min\; \tilde \beta,
\text{ s.t. }
&\tilde A_j(Y)+\tilde B_j+\sum_{n=0}^N r_n^{(j,-)} H^{(j)}_n+\tilde \beta e_{d+1}e_{d+1}^T\succeq 0,\\
&-\tilde A_j(Y)-\tilde B_j+\sum_{n=0}^N r_n^{(j,+)} H^{(j)}_n+\tilde \beta e_{d+1}e_{d+1}^T\succeq 0,\\
&r^{(j,-)}\geq 0, r^{(j,+)}\geq 0, j\in[p], 
\end{aligned}
\end{equation}
where the variables are $\tilde \beta$ and $\{r^{(j,+)},r^{(j,-)}\}_{j=1}^p$. 
Let $\tilde \beta_2$ be the optimal value of the above problem. %{\color{red} the above problem is a minimization problem, how can you have $\beta_\mathrm{max}$? again, what is the relation between (21) and (20) or (16)?} 
For $\tilde \beta \geq \tilde \beta_2$, as $\mfY$ is feasible for the constraints in \eqref{cvx_nn:d_relax}, the optimal value of the relaxed dual problem \eqref{cvx_nn:d_relax} is $0$. In short, only when $\tilde \beta \in [\tilde \beta_1, \tilde \beta_2]$, %{\color{red} this should be $[\beta_1, \beta_2]$, right?} 
the variational problem \eqref{cvx_nn:d_relax} is non-trivial. To ensure that solving the relaxed dual problem \eqref{cvx_nn:d_relax} gives a good approximation of the Wasserstein gradient direction, we shall avoid choosing $\tilde \beta$ either too small or too large. %Therefore, solving this problem will most likely yield a good approximation of the Wasserstein gradient direction. {\color{red} why does this most likely yield a good approximation? ``most likely" in what sense? how to measure ``good"? how does this lower and upper bound help you to choose the parameter in the algorithm? add necessary explanations for them or add a pointer to the appendix if they are explained there}

% {\color{red} Again I think one thing missing is the complexity to solve the SDP problem, at least its dependence on $N, p, d$. the influence of subsampling of $[p]$ should be illustrated by numerical example or referred to some other papers on the sampling number choice. meanwhile, we used dimension reduction in $d$, which should also be stated in the complexity analysis or in the algorithm, with its necessity/influence either illustrated or referred to some other papers.}

\section{Proofs}
\subsection{Proof of Proposition \ref{prop:wgd_var}}

\begin{proof}
We first note that 
\begin{equation}
\begin{aligned}
    &\frac{1}{2}\int \|\nabla \Phi-\nabla \log \rho+\nabla \log \pi\|_2^2 \rho dx\\
    =&\frac{1}{2} \int \|\nabla \Phi\|_2^2 \rho dx+\int \lra{\nabla \log\pi-\nabla \log \rho ,\nabla \Phi}\rho dx\\
    &+\frac{1}{2}\int \|\nabla \log \rho-\nabla \log \pi\|_2^2\rho dx.
\end{aligned}
\end{equation}
We notice that the term $\frac{1}{2}\int \|\nabla \log \rho-\nabla \log \pi\|_2^2\rho dx$ does not depend on $\Phi$. Utilizing the integration by parts, we can compute that
\begin{equation}
\begin{aligned}
    \int \lra{\nabla \log \rho, \nabla \Phi} \rho dx
    =&\int \lra{\frac{\nabla\rho}{\rho},\nabla \Phi } \rho dx\\
    =&\int \lra{\nabla \rho, \nabla \Phi}dx\\
    =&-\int \Delta \Phi \rho dx.
\end{aligned}
\end{equation}
Therefore, the variational problem \eqref{equ:wgd} is equivalent to
\begin{equation}
    \inf_{\Phi\in C^\infty(\mbR^d)} \frac{1}{2}\int \|\nabla \Phi\|_2^2 \rho dx+\int \lra{\nabla \log\pi,\nabla \Phi}\rho dx+\int \Delta \Phi \rho dx. 
\end{equation}
By restricting the domain $C^\infty(\mbR^d)$ to $\mcH$, we complete the proof. 
\end{proof}

\subsection{Proof of Proposition \ref{prop:reform}}
\begin{proof}
Suppose that $\hat w_i=\beta_i^{-1} w_i$ and $\hat \alpha_i=\beta_i^2\alpha_i$, where $\beta_i>0$ is a scale parameter for $i\in[m]$. Let $\btheta'=\{(\hat w_i,\hat \alpha_i)\}^m_{i=1}$. We note that
\begin{equation}
    \hat \alpha_i\hat w_i\psi'(\hat w_i^Tx_n) = \beta_i\alpha_i w_i\psi'\pp{\beta_i^{-1} w_i^Tx_n}=\alpha_i w_i\psi'( w_i^Tx_n),
\end{equation}
and
\begin{equation}
    \hat \alpha_i \| \hat w_i\|_2^2\psi''(\hat w_i^Tx_n)=\alpha_i\|w_i\|_2^2\psi''(\hat w_i^Tx_n)=\alpha_i\|w_i\|_2^2\psi''(w_i^Tx_n).
\end{equation}
This implies that $\Phi_{\btheta}(x)=\Phi_{\btheta'}(x)$ and $\nabla\cdot \Phi_{\btheta}(x)=\nabla\cdot\Phi_{\btheta'}(x)$. For the regularization term $R(\btheta)$, we note that
\begin{equation}
\begin{aligned}
\|\hat w_i\|_2^3+\|\hat \alpha_i\|_2^3=&\beta_i^6 |\alpha_i|^3+\beta_i^{-3} \|w_i\|_2^3 \\
=&  \beta_i^6 |\alpha_i|^3+\frac{1}{2}\beta_i^{-3} \|w_i\|_2^3+\frac{1}{2}\beta_i^{-3} \|w_i\|_2^3\\
=& 3\cdot 2^{-2/3} \|w_i\|_2^2|\alpha_i|.
\end{aligned}
\end{equation}
The optimal scaling parameter is given by $\alpha_i=2^{-1/9}\frac{\|w_i\|_2^{1/3}}{|\alpha_i|_1^{1/3}}$. As the scaling operation does not change $\|w_i\|_2^2|\alpha_i|$, we can simply let $\|w_i\|_2=1$. Thus, the regularization term $\frac{\beta}{2}R(\btheta)$ becomes $\frac{\tilde \beta}{N}\sum_{i=1}^m \|u_i\|_1$. This completes the proof. 
\end{proof}

\subsection{Proof of Proposition \ref{prop:dual}}
\begin{proof}
Consider the Lagrangian function
\begin{equation}
\begin{aligned}
L(Z,W,\alpha, \Lambda) = &\frac{1}{2}\|Z\|_F^2+\sum_{n=1}^N\sum_{i=1}^m \alpha_i \|w_i\|_2^2 \psi''(w_i^Tx_n) +\tr(Y^TZ)+\tilde \beta \|\alpha\|_1\\
&+\sum_{n=1}^N \lambda_n^T\pp{z_n-\sum_{i=1}^m\alpha_iw_i\psi'(x_n^Tw_i)}\\
=&\tilde \beta \|\alpha\|_1+\sum_{i=1}^m \alpha_i \sum_{n=1}^N\pp{\|w_i\|_2^2 \psi''(w_i^Tx_n)-\lambda_n^Tw_i\psi'(x_m^Tw_i)}\\
&+\frac{1}{2}\|Z\|_F^2+\tr((Y+\Lambda)^TZ).
\end{aligned}
\end{equation}
For fixed $W$, the constraints on $Z$ and $\alpha$ are linear and the strong duality holds. Thus, we can exchange the order of $\min_{Z,\alpha}$ and $\max_{\Lambda}$. Thus, we can compute that
\begin{equation}
\begin{aligned}
&\min_{Z,W,\alpha}\max_{\Lambda} L(Z,W,\alpha, \Lambda)\\
=&\min_{W}\max_{\Lambda} \min_{\alpha,Z}L(Z,W,\alpha, \Lambda)\\
=&\min_{W}\max_{\Lambda} \min_{\alpha,Z}\tilde \beta \|\alpha\|_1+\sum_{i=1}^m \alpha_i \sum_{n=1}^N\pp{\|w_i\|_2^2 \psi''(w_i^Tx_n)-\lambda_n^Tw_i\psi'(x_m^Tw_i)}+\frac{1}{2}\|Z\|_F^2+\tr((Y+\Lambda)^TZ)\\
=&\min_{W}\max_{\Lambda} -\frac{1}{2}\|\Lambda+Y\|_F^2+\sum_{i=1}^m \mbI\pp{\max_{w_i:\|w_i\|_2\leq 1}\left|\sum_{n=1}^N\|w_i\|_2^2 \psi''(w_i^Tx_n)-y_n^Tw_i\psi'(x_n^Tw_i)\right|\leq \tilde \beta}.
\end{aligned}
\end{equation}
By exchanging the order of $\min$ and $\max$, we can derive the dual problem:
\begin{equation}
\begin{aligned}
&\max_{\Lambda}\min_{W} -\frac{1}{2}\|\Lambda+Y\|_F^2+\sum_{i=1}^m \mbI\pp{\max_{w_i:\|w_i\|_2\leq 1}\left|\sum_{n=1}^N\|w_i\|_2^2 \psi''(w_i^Tx_n)-y_n^Tw_i\psi'(x_n^Tw_i)\right|\leq \tilde \beta}\\
=&\max_{\Lambda} -\frac{1}{2}\|\Lambda+Y\|_F^2\text{ s.t. } \max_{w_i:\|w_i\|_2\leq 1}\left|\sum_{n=1}^N\|w_i\|_2^2 \psi''(w_i^Tx_n)-y_n^Tw_i\psi'(x_n^Tw_i)\right|\leq \tilde \beta, i\in[m]\\
=&\max_{\Lambda} -\frac{1}{2}\|\Lambda+Y\|_F^2\text{ s.t. } \max_{w:\|w\|_2\leq 1}\left|\sum_{n=1}^N\|w\|_2^2 \psi''(w^Tx_n)-y_n^Tw\psi'(x_n^Tw)\right|\leq \tilde \beta, i\in[m]\\
\end{aligned}
\end{equation}
This completes the proof.
\end{proof}

\subsection{Proof of Proposition \ref{prop:dual_relax}}
\begin{proof}
% \begin{remark}
% We note that the optimal rank-1 solution will take the form 
% \begin{equation}
% Z=\tilde w \tilde w^T,\quad \tilde w=\bmbm{w\\1}.
% \end{equation}
% \end{remark}
Based on the hyper-plane arrangements $D_1,\dots,D_p$, the dual constraint is equivalent to that for all $j\in[p]$,
\begin{equation}
\left|2\tr(D_j)\|w\|_2^2-2w^T\Lambda^TD_jXw\right|\leq \tilde \beta
\end{equation}
holds for all $w\in\mbR^d$ satisfying $\|w\|_2\leq 1, (2D_j-I)Xw\geq 0$. This is equivalent to say that for all $j\in[p]$

\begin{align}
    -\tilde \beta\geq &\min 2\tr(D_j)\|w\|_2^2-2w^T\Lambda^TD_jXw,\label{dual_sub:min}\\
    &\text{ s.t. }\|w\|_2\leq 1, 2(D_j-I)Xw\geq 0,\notag\\
    \tilde \beta\leq &\max 2\tr(D_j)\|w\|_2^2-2w^T\Lambda^TD_jXw,\notag\\
    &\text{ s.t. }\|w\|_2\leq 1, 2(D_j-I)Xw\geq 0.\notag
\end{align}

From a convex optimization perspective, the natural idea to interpret the constraint \eqref{dual_sub:min} is to transform the minimization problem into a maximization problem. We can rewrite the minimization problem in \eqref{dual_sub:min} as a trust region problem with inequality constraints:
\begin{equation}\label{trust_region}
\begin{aligned}
\min_{w\in\mbR^d}\; & w^T\pp{B_j+A_j(\Lambda)}w, \\
\text{ s.t. }&\|w\|_2\leq 1, (2D_j-I)Xw\geq 0.
\end{aligned}
\end{equation}

As the problem \eqref{trust_region} is a convex problem, by taking the dual of \eqref{trust_region} w.r.t. $w$, we can transform \eqref{trust_region} into a maximization problem. However, as \eqref{trust_region} is a trust region problem with inequality constraints, the dual problem of \eqref{trust_region} can be very complicated. According to \citep{jeyakumar2014trust}, %the optimal value of \eqref{trust_region} is lower-bounded by the optimal value of an SDP formulation. 
% Then, 
the optimal value of the problem \eqref{trust_region} is bounded by the optimal value of the following SDP
\begin{equation}\label{sdp:dual}
\begin{aligned}
\min_{Z\in\mbS^{d+1}}\;&\tr((\tilde A_j(\Lambda) +\tilde B_j)Z), \\
\text{ s.t. } &\tr(H^{(j)}_nZ)\leq 0, n=0,\dots,N, \\
&Z_{d+1,d+1}=1, Z\succeq 0.
\end{aligned}
\end{equation}
from below.
\begin{lemma}\label{lem:SDP_d}
The dual problem of SDP \eqref{sdp:dual} takes the form
\begin{equation}\label{sdp:bidual}
\max -\gamma, \text{ s.t. }S = \tilde A_j(\Lambda)+\tilde B_j+\sum_{n=0}^N r_n H^{(j)}_n+\gamma e_{d+1}e_{d+1}^T,r\geq 0, S\succeq 0,
\end{equation}
in variables $r=\bmbm{r_0\\\vdots\\r_N}\in\mbR^{N+1}$ and $\gamma\in\mbR$.
\end{lemma}
\begin{proof}
Consider the Lagrangian 
\begin{equation}
L(Z,r,\gamma) =  \tr((\tilde A_j(y)+\tilde B_j) Z)+\sum_{n=0}^N r_n\tr(H_n^{(j)}Z)+\gamma(\tr(Ze_{d+1}e_{d+1}^T)-1),
\end{equation}
where $r\in\mbR^{N+1}_+$ and $\gamma\in \mbR$. By minimizing $L(Z,r,\gamma)$ w.r.t. $Z\in \mbS^{d+1}_{+}$, we derive the dual problem \eqref{sdp:bidual}.
\end{proof}

The constraints on $\Lambda$ in the dual problem \eqref{cvx_nn:d} include that the optimal value of \eqref{sdp:dual} is bounded from below by $-\tilde \beta$. According to Lemma \ref{lem:SDP_d}, this constraint is equivalent to that %the optimal value of \eqref{sdp:bidual} is not less than $-\tilde \beta$. Equivalently, 
there exist $r\in \mbR^{N+1}$ and $\gamma$ such that 
\begin{equation}
    -\gamma \geq -\tilde \beta, S = \tilde A_j(\Lambda)+\tilde B_j+\sum_{n=0}^N r_n H^{(j)}_n+\gamma e_{d+1}e_{d+1}^T,r\geq 0, S\succeq 0.
\end{equation}
As $e_{d+1}e_{d+1}^T$ is positive semi-definite, the above condition on $\Lambda$ is also equivalent to that there exist $r\in \mbR^{N+1}$ such that
\begin{equation}
    \tilde A_j(\Lambda)+\tilde B_j+\sum_{n=0}^N r_n H^{(j)}_n+\tilde \beta e_{d+1}e_{d+1}^T\succeq 0, r\geq 0.
\end{equation}
Therefore, the following convex set of $\Lambda$
\begin{equation}
\begin{aligned}
\Big\{\Lambda:\tilde A_j(\Lambda)+\tilde B_j+\sum_{n=0}^N r_n^{(j,-)} H^{(j)}_n+\tilde \beta e_{d+1}e_{d+1}^T\succeq 0,\; r^{(j,-)}\geq 0\Big\}
\end{aligned}
\end{equation}
is a subset of the set of $\Lambda$ satisfying the dual constraints
\begin{equation}
\left\{\Lambda:\min_{\|w\|_2\leq 1, (2D_j-I)w\geq 0}w^T\pp{B_j+A_j(\Lambda)}w\geq -\tilde \beta\right\}
\end{equation}
On the other hand, the constraint on $\Lambda$
\begin{equation}
    \max_{\|w\|_2\leq 1, (2D_j-I)w\geq 0}w^T\pp{B_j+A_j(\Lambda)}w\leq \tilde \beta
\end{equation}
is equivalent to
\begin{equation}
    \min_{\|w\|_2\leq 1, (2D_j-I)w\geq 0}-w^T\pp{B_j+A_j(\Lambda)}w\geq -\tilde \beta.
\end{equation}
By applying the previous analysis on the above trust region problem, the following convex set of $\Lambda$
\begin{equation}
\begin{aligned}
\Big\{\Lambda:-\tilde A_j(\Lambda)-\tilde B_j+\sum_{n=0}^N r_n^{(j,+)} H^{(j)}_n+\tilde \beta e_{d+1}e_{d+1}^T\succeq 0,\; r^{(j,+)}\geq 0\Big\}
\end{aligned}
\end{equation}
is a subset of the set of $\Lambda$ satisfying the dual constraints
\begin{equation}
\left\{\Lambda:\max_{\|w\|_2\leq 1, (2D_j-I)w\geq 0}w^T\pp{B_j+A_j(\Lambda)}w\leq \tilde \beta\right\}. 
\end{equation}
Therefore, replacing the dual constraint $\max_{w:\|w\|_2\leq 1}\left|\sum_{n=1}^N \|w\|_2^2\psi''(w^Tx_n)-y_n^Tw\psi'(x_n^Tw)\right|\leq \tilde \beta$ by
\begin{equation}
\begin{aligned}
&\tilde A_j(\Lambda)+\tilde B_j+\sum_{n=0}^N r_n^{(j,-)} H^{(j)}_n+\tilde \beta e_{d+1}e_{d+1}^T\succeq 0, j\in[p],\\
&-\tilde A_j(\Lambda)-\tilde B_j+\sum_{n=0}^N r_n^{(j,+)} H^{(j)}_n+\tilde \beta e_{d+1}e_{d+1}^T\succeq 0, j\in[p],\\
&r^{(j,-)}\geq 0, r^{(j,+)}\geq 0, j\in[p].
\end{aligned}
\end{equation}
we obtain the relaxed dual problem. As its feasible domain is a subset of the feasible domain of the dual problem, the optimal value of the relaxed dual problem gives a lower bound for the optimal value of the dual problem.
\end{proof}

\subsection{Proof of Proposition \ref{prop:bidual}}
\begin{proof}
Consider the Lagrangian function
\begin{equation}
\begin{aligned}
L(\Lambda,\mfr, \mfS) =& -\frac{1}{2}\|\Lambda+Y\|_2^2-\sum_{j=1}^p \tr\pp{S^{(j,-)}\pp{\tilde A_j(\Lambda)+\tilde B_j+\sum_{n=0}^N r_n^{(j,-)} H^{(j)}_n+\frac{\tilde \beta}{2}e_{d+1}e_{d+1}^T}}\\
&-\sum_{j=1}^p\tr\pp{S^{(j,+)}\pp{-\tilde A_j(\Lambda)-\tilde B_j+\sum_{n=0}^N r_n^{(j,+)} H^{(j)}_n+\frac{\tilde \beta}{2} e_{d+1}e_{d+1}^T}},
%&+\sum_{j=1}^p\pp{\pp{ r^{(j,-)}}^Tq^{(j,-)} +\pp{ r^{(j,+)}}^Tq^{(j,+)} },
\end{aligned}
\end{equation}
where we write 
\begin{equation}
\begin{aligned}
\mfr=&\pp{r^{(1,-)},\dots,r^{(p,-)},r^{(1,+)},\dots,r^{(p,+)}}\in \pp{\mbR^{N+1}}^{2p},\\
\mfS=&\pp{S^{(1,-)},\dots,S^{(p,-)},S^{(1,+)},\dots,S^{(p,+)}}\in \pp{\mbS^{d+1}_+}^{2p}.
%\mfq=&\pp{q^{(1,-)},\dots,q^{(p,-)},q^{(1,+)},\dots,q^{(p,+)}}\in \pp{\mbR^{N+1}_+}^{2p}.
\end{aligned}
\end{equation}
Here we write $\mbS^{d+1}_{+}=\{S\in \mbS^{d+1}|S\succeq 0\}$. By maximizing w.r.t. $\Lambda$ and $\mfr$, we derive the bi-dual problem \eqref{cvx_nn:bidual}.

\subsection{Proof of Theorem \ref{thm:map}}
% \subsection{Map from the solutions of the primal problem to the solutions of relaxed-bidual problem}\label{app:map}
%We show that any feasible solution of the primal problem \eqref{cvx_nn:p} can be mapped to a feasible solution of \eqref{cvx_nn:bidual} and they have the same objective value. 
Suppose that $(Z,W,\alpha)$ is a feasible solution to \eqref{equ:wgd_reg_rep}. Let $D_{j_1},\dots,D_{j_k}$ be the enumeration of $\{\diag(\mbI(Xw_i\geq 0))|i\in[m]\}$. %We can assume that $\mbI(Xw_i\geq 0)$ are different for $i\in[m]$. Let $D_{j_i}=\diag(\mbI(Xw_i\geq 0))$ for $i\in[m]$. 
For $i\in[k]$, we let
\begin{equation}
    S^{(j_i,+)}=\sum_{l:\alpha_l\geq 0, \diag(\mbI(Xw_l\geq 0))=D_{j_i}}\alpha_l\bmbm{w_lw_l^T&w_l\\w_l^T&1}, S^{(j_i,-)}=0,
\end{equation}
and
\begin{equation}
    %S^{(j_i,+)}=0, S^{(j_i,-)}=-\alpha_i\bmbm{w_iw_i^T&w_i\\w_i^T&1}.
    S^{(j_i,+)}=0, S^{(j_i,-)}=-\sum_{l:\alpha_l< 0, \diag(\mbI(Xw_l\geq 0))=D_{j_i}}\alpha_l\bmbm{w_lw_l^T&w_l\\w_l^T&1}.
\end{equation}
For $j\notin \{j_1,\dots,j_k\}$, we simply set $S^{(j,+)}=0, S^{(j,-)}=0$. As $\|w_i\|_2\leq 1$ and $D_{j_i}=\mbI( X w_i\geq 0)$, we can verify that $\tr(S^{(j,-)}H_n^{(j)})\leq 0, \tr(S^{(j,+)}H_n^{(j)})\leq 0$ are satisfied for $j=j_1,\dots,j_m$ and $n=0,1,\dots,N$. This is because for $n=0$, as $H_0^{(j_i)}=\bmbm{I_d&0\\0&-1}$, it follows that
\begin{equation}
\begin{aligned}
\tr(S^{(j_i,+)}H_0^{(j_i)})=&\sum_{l:\alpha_l\geq 0, \diag(\mbI(Xw_l\geq 0))=D_{j_i}}\alpha_l(\|w_l\|^2-1)\leq 0,\\
\tr(S^{(j_i,-)}H_0^{(j_i)})=&-\sum_{l:\alpha_l< 0, \diag(\mbI(Xw_l\geq 0))=D_{j_i}}\alpha_l(\|w_l\|^2-1)\leq 0.
\end{aligned}
\end{equation}
For $n=1,\dots,N$, we have
\begin{equation}
\begin{aligned}
\tr(S^{(j_i,+)}H_0^{(j_i)})=&\sum_{l:\alpha_l\geq 0, \diag(\mbI(Xw_l\geq 0))=D_{j_i}}2\alpha_l(1-2(D_{j_i})_{nn})x_n^Tw_l\leq 0,\\
\tr(S^{(j_i,-)}H_0^{(j_i)})=&-\sum_{l:\alpha_l< 0, \diag(\mbI(Xw_l\geq 0))=D_{j_i}}\alpha_l(1-2(D_{j_i})_{nn})x_n^Tw_l\leq 0.
\end{aligned}
\end{equation}

Based on the above transformation, we can rewrite the bidual problem in the form of the primal problem \eqref{cvx_nn:p}. For $S\in\mbS^{d+1}$, we note that
\begin{equation*}
\begin{aligned}
    &\tr(S\tilde A_j(\Lambda)) \\
    =& -\tr((\Lambda^TD_j X+X^TD_j\Lambda )S_{1:d,1:d}) \\
    =&-2\tr(\Lambda^T D_jXS_{1:d,1:d}), 
\end{aligned}
\end{equation*}
where $S_{1:d,1:d}$ denotes the $d\times d$ block of $S$ consisting the first $d$ rows and columns. This implies that $\tilde A_j^*(S)=-2D_j X S_{1:d, 1:d}$. Hence, we have 
$$
\tilde A_{j_i}(S^{(j_i,+)}-S^{(j_i,-)})=-\sum_{l:\diag(\mbI(Xw_l\geq 0)}2\alpha_l D_{j_i} Xw_lw_l^T=-\sum_{l:\diag(\mbI(Xw_l\geq 0)}2\alpha_l(Xw_l)_+w_l^T.
$$
Therefore, we have 
$$
\sum_{j=1}^p \tilde A_j^*(S^{(j,-)}-S^{(j,+)})=2\sum_{i=1}^m \alpha_i(Xw_i)_+w_i^T.%=Z.
$$
As $n$-th row of $Z$ satisfies that $z_n=2\sum_{i=1}^m\alpha_iw_i(x_n^Tw_i)_+$, this implies that 
$$
Z=2\sum_{i=1}^m \alpha_i(Xw_i)_+w_i^T=\sum_{j=1}^p \tilde A_j^*(S^{(j,-)}-S^{(j,+)}).
$$
Hence $(Z,\{(S^{(j,-)},(S^{(j,-)}\}_{j=1}^p)$ is feasible to the relaxed bi-dual problem \eqref{cvx_nn:bidual}. 

We can also compute that 
$$
\sum_{j=1}^p \tr(\tilde B_j(S^{(j,+)}-S^{(j,-)}))=2\sum_{i=1}^m \alpha_i\sum_{n=1}^N\mbI(x_n^Tw_i\geq 0)\|w_i\|_2^2,
$$
and 
$$
\sum_{j=1}^p\tr\pp{(S^{(j,+)}+S^{(j,-)})e_{d+1}e_{d+1}^T }=\sum_{i=1}^m|\alpha_i|.
$$
Thus, the primal problem \eqref{cvx_nn:p} with $(Z,W,\alpha)$ and the relaxed bi-dual problem \eqref{cvx_nn:bidual} with $(Z,\{(S^{(j,-)},(S^{(j,-)}\}_{j=1}^p)$ have the same objective value. 
\end{proof}

\end{document}